\documentclass[letterpaper, 10 pt, conference]{ieeeconf}  

\IEEEoverridecommandlockouts                              

\usepackage{algorithm}
\usepackage{algorithmic}

\usepackage[utf8]{inputenc}
\usepackage{mathtools}
\usepackage{amsmath,amsfonts,bm}
\usepackage{multirow}

\usepackage{subcaption}

\usepackage{color}
\definecolor{darkgreen}{rgb}{0,0.5,0}
\definecolor{azure(colorwheel)}{rgb}{0.0, 0.5, 1.0}

\newtheorem{theorem}{Theorem}

\usepackage{booktabs} 

\title{\LARGE \bf
Scalable Safety-Critical Policy Evaluation \\
with Accelerated Rare Event Sampling
}
\author{
    Mengdi Xu$^{1}$, 
    Peide Huang$^{1}$, 
    Fengpei Li$^{2, 3}$, 
    Jiacheng Zhu$^{1}$, 
    Xuewei Qi$^{4}$, 
    Kentaro Oguchi$^{4}$, \\
    Zhiyuan Huang$^{1,5}$, 
    Henry Lam$^{2}$, 
    Ding Zhao$^{1}$
\thanks{$^{1}$ Mengdi Xu, Peide Huang, Jiacheng Zhu, Zhiyuan Huang, and Ding Zhao are with Carnegie Mellon University.
        {\tt\small mengdixu, peideh, jzhu4, dingzhao@andrew.cmu.edu}}%
\thanks{$^{2}$ Fengpei Li and Henry Lam are with Columbia University.
        {\tt\small khl2114@columbia.edu }}%
\thanks{$^{3}$ Fengpei Li is also with Morgan Stanley Machine Learning Research.
        {\tt\small fl2412@columbia.edu }}%
\thanks{$^{4}$ Xuewei Qi and Kentaro Oguchi are with Toyota Motor North America R$\&$D.
        {\tt\small tony.qi, kentaro.oguchi@toyota.com}}%
\thanks{$^{5}$ Zhiyuan Huang is also with Tongji University.
        {\tt\small zhuang2@andrew.cmu.edu}}%
}

\begin{document}

\maketitle
\thispagestyle{empty}
\pagestyle{empty}

\begin{abstract}
Evaluating rare but high-stakes events is one of the main challenges in obtaining reliable reinforcement learning policies, especially in large or infinite state/action spaces where limited scalability dictates a prohibitively large number of testing iterations. On the other hand, a biased or inaccurate policy evaluation in a safety-critical system could potentially cause unexpected catastrophic failures during deployment. This paper proposes the Accelerated Policy Evaluation (APE) method, which simultaneously uncovers rare events and estimates the rare event probability in Markov decision processes. The APE method treats the environment nature as an adversarial agent and learns towards, through adaptive importance sampling, the zero-variance sampling distribution for the policy evaluation. Moreover, APE is scalable to large discrete or continuous spaces by incorporating function approximators. We investigate the convergence property of APE in the tabular setting. Our empirical studies show that APE can estimate the rare event probability with a smaller bias while only using orders of magnitude fewer samples than baselines in multi-agent and single-agent environments.
\end{abstract}


\section{INTRODUCTION}

There has been a growing interest in applying reinforcement learning (RL) to complex high-stakes robotics problems, including controlling autonomous vehicles and healthcare assistant robots \cite{kalra2016driving, chen2021context}. 
As one can expect, before deployment of any RL policy, such safety-critical applications often require an accurate policy evaluation because inaccurate estimations could lead to false optimism or even catastrophic consequences. However, a main difficulty for policy evaluation in these settings is that we are crucially interested in assessing certain rare but high-stake events in safety-critical systems.
The \textit{de facto} standard evaluation method is the vanilla Monte Carlo (MC) which, in a multitude of practical settings of interest, requires a prohibitively large number of testing before the evaluation can be deemed statistically valid \cite{rubinstein2016simulation}.
For example, \cite{kalra2016driving} shows that, in order to demonstrate the safety of self-driving cars, the number of miles that needs to be clocked in is technically in the hundreds of billions.
Consequently, when such large-volume testing is high-stakes, costly, or time-consuming, a growing number of studies have been developed to accelerate the estimation of rare event probability as the policy evaluation metric \cite{zhao2016accelerated, corso2020scalable}. 
However, existing literature typically either fails to exploit the sequential, interactive nature of tasks by confining to specific failure-causing initial conditions \cite{uesato2018rigorous, o2018scalable} or only focuses on small state/action spaces due to the curse of dimensionality \cite{desai2001markov, ahamed2006adaptive}.

\begin{figure}
\centering
\includegraphics[width=7cm]{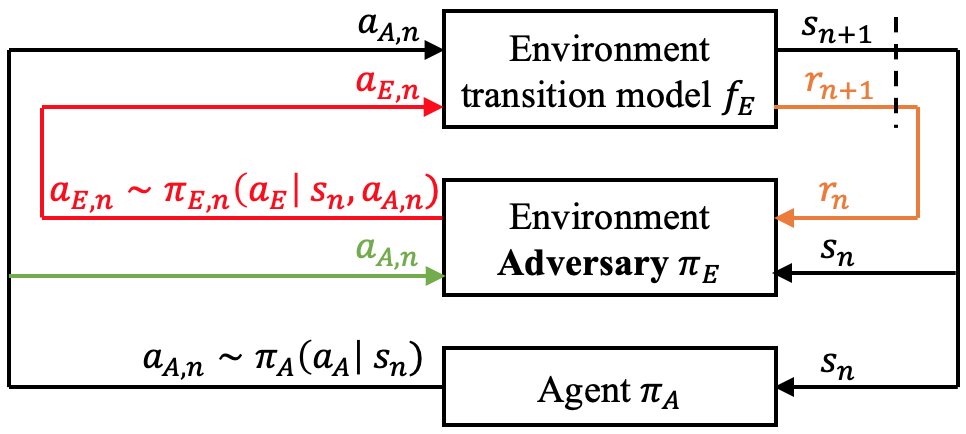}
\caption{
The interactions among the environment transition model, the agent and the environment adversary introduced by our proposed accelerated policy evaluation method. 
The reward signal $r$ is for updating the adversary policy $\pi_E$. The subscript $n$ denotes the time step.
\vspace{-0.3in}
}
\label{fig:ape_pipeline}
\end{figure}

In this paper, we develop an efficient and scalable policy evaluation method for estimating rare event probabilities in sequential decision-making tasks, which expedites the policy evaluation in large or continuous state/action spaces. In particular, we specify a rare event set and estimate the probability of entering the set under a given RL policy in an environment of interest.
We focus on episodic settings with \textit{an adjustable grey-box testing ground} that simulates the environment (i.e., a simulator with adjustable disturbances or vehicle policies), which is the common practice in third-party evaluations or in curriculum design \cite{uesato2018rigorous, dennis2020emergent}.
We summarize our major contributions as follows:
\begin{itemize}
    \item We propose the \emph{Accelerated Policy Evaluation} (APE) method, suitable for Markov decision processes (MDPs) and motivated by adaptive importance sampling \cite{ahamed2006adaptive}.
    In a novel way, APE introduces an adversary that controls the environment transition at each step and gradually exploits sequential interactions that cause rare events. We further theoretically establish the convergence property of APE in the tabular setting.
    \item We propose a scalable implementation of APE and denote it as \emph{APE\_scalable}. \emph{APE\_scalable} incorporates two function approximators: a Gaussian Process (GP) \cite{rasmussen2003gaussian} to estimate rare event probability with high data efficiency, and a conditional normalizing flow model \cite{papamakarios2017masked} to expressively represent the adversary's policy. We further design a two-timescale gradient-based updating rule for GP to adapt to the rare event setting. 
    \item We empirically show that APE\_scalable estimates the rare event probability with a smaller bias and orders of magnitude fewer samples than baselines, in driving scenarios \cite{highway-env} and in a LunarLander control task \cite{brockman2016openai}.
\end{itemize}


\section{Related Work}
Motivated by the inherent difficulty of assessing rare, safety-critical events in lab or field tests, there has been a growing literature to maximize the sampling efficiency of rare event probability estimation.
Two powerhouses are the importance sampling method \cite{zhao2016accelerated, zhao2017accelerated, arief2020deep} and the multi-level splitting method \cite{ cerou2007adaptive}. 
However, little literature focuses on MDP settings. Recent developments in \cite{uesato2018rigorous, ding2021multimodal} adjust the environment setting at the beginning of each episode while ignoring the intermediate interactive steps. However, our proposed APE method is designed to leverage the sequential nature of tasks and hence is more suitable in the RL settings.
The APE is similar to the adaptive stochastic approximation (ASA) for Markov chains \cite{ahamed2006adaptive}. Nonetheless, ASA restricts to discrete state space and is not incorporated with the decision process, thus handicapping its applicability for sophisticated decision-making agents such as continuous-space RL. Our setting is similar to \cite{corso2020scalable} because both utilize a dynamic programming paradigm. However, \cite{corso2020scalable} relies on discretizing and decomposing the simulation scene to make it scalable.

APE is within the regime of safe validation \cite{corso2020survey} and specifically focuses on finding the distribution of failure-causing disturbances and estimating the failure rate. Adversarial evaluation \cite{uesato2018rigorous, werpachowski2019detecting}, mainly motivated by attacks on RL agents,
also aims to find failure cases but in a distribution-free manner. 
Instead of searching for worst-case scenarios as in adversarial attacks \cite{eykholt2018robust}, APE aims to directly find those rare but possible (failure) cases when the agent interacts with the original environment of interest.

APE utilizes updating rules involving dynamic programming and importance weights over changing policies. The rules share a similar form with off-policy evaluation \cite{xie2019towards} methods
including Retrace~\cite{munos2016safe}, V-trace~\cite{espeholt2018impala}, and importance resampling~\cite{schlegel2019importance}. 
Concretely, the environment's ground truth policy $\pi_{E, gt}$ and the proposed adversary policy $\pi_E$ in APE are analogous to the evaluation policy and behavior policy in off-policy evaluation \cite{cotta2019off}, respectively. 



\section{Rare Event Probability as Policy Value}
A general MDP consists of a tuple $(\mathcal S, \mathcal A, p, r)$. $\mathcal S $ and $ \mathcal A$ denote the state and action space. $s'$ denotes the proceeding state of $s$. The stochastic transition model is $p(s'|s,a): s,s'\in\mathcal S, a \in\mathcal A$
and the one-step reward is $r(s,s'), s,s' \in\mathcal S$. 
In the rare event setting, we consider the state space $\mathcal S = \mathcal T \cup \mathcal I $ is a partition of terminal states $\mathcal T$ and interior states $\mathcal I$, $\mathcal T \cap \mathcal I = \emptyset$. We assume $\mathcal T$ is reachable from any states in $\mathcal I$ for all policies under consideration and 
the episode length $\tau$ is finite almost surely. The rare event set is denoted by $\mathcal R$ and $\mathcal R\subset \mathcal T $. 

We formulate the rare event probability (or the probability of hitting $\mathcal{R}$ before $\mathcal T \backslash \mathcal R$) starting from an initial state $s_0 \in \mathcal I$ in an environment of interest, as the expected undiscounted total reward until termination.
With $v^\star$ as the true value, we have the following Bellman equation holds:
\begin{align}
    & v^\star(s_0) \coloneqq  \mathbb{E}_{p_{E, gt}, \pi_A} \Big[ \sum_{n=0}^{\tau-1} r(s_{n}, s_{n+1}) \Big] = \label{eq:rare_prob_definition}
    \\ &\sum_{a_A \in \mathcal{A}_A} \pi_A(a_A|s_0) \sum_{s_1 \in \mathcal{S} }p_{E,gt}(s_1 | s_0, a_A) \Big[ r(s_0, s_1) + v^\star(s_1) \Big],
    \nonumber
\end{align}
where $\pi_A$ is the given Markov policy (typically obtained through training). $p_{E,gt}$ is the ground truth transition probability of the environment of interest. $n$ is the timestep. Furthermore, $a_A \sim \pi_A$ is the agent action, and $\mathcal{A}_A$ is the agent action space. The one-step reward is an indicator function with $r(s, s') = 1$ if $s' \in \mathcal{R}$ and 0 otherwise. 
We set $v^\star(s)=0$ for $s\in\mathcal T$.
The goal is to estimate  $v^\star(s), s\in\mathcal I$. 

In settings with large/continuous state space, approximating values in \eqref{eq:rare_prob_definition} is expensive or infeasible. Moreover, it is hard to guarantee a satisfactory level of estimation error because the considered rare event region $\mathcal R$ is not visited frequently enough. In order to frequently visit $\mathcal R$ in these high or infinite-dimensional situations, one is motivated to consider the importance sampling \cite{asmussen2007stochastic} to modify the environment transition probability, which achieves variance reduction by approaching a target (zero-variance) distribution and learns to generate rare events efficiently.

\section{Accelerated Policy Evaluation (APE)}
\label{sec:ape}

In this section, we introduce the proposed APE paradigm.
In MDPs, the environment transition probability relates to both agent A's policy $\pi_A$ and the environment's transition probability $p_E(s' | a_A, s)$, $p_{s, s'} = \sum_{a_A \in \mathcal{A}_A} \pi_A(a_A|s) p_E(s' | a_A, s)$. Therefore, $v^\star(s_0)$ can be reformulated by rolling out trajectories with the newly proposed environment transition probability $p_E^{(n)}$ and agent policy $\pi_A^{(n)}$ at step $n$ as follows:
\begin{align}
    v^\star(s_0) &= \mathbb{E}_{p_{E}^{(n)}, \pi_A^{(n)}} \Big[ \sum_{n=0} \rho_n (s_n, s_{n+1}, a_A) \cdot r(s_{n}, s_{n+1}) \Big] \nonumber
    \\ &= \sum_{a_A \in \mathcal{A}_A} \pi_A^{(0)}(a_A|s_0) \sum_{s_1 \in \mathcal{S} }p_{E}^{(0)}(s_1 | s_0, s_A) \cdot \nonumber \\ 
    & \rho_0(s_0, s_1, a_A)  \Big[ r(s_0, s_1) + v^\star(s_1) \Big], \label{eq:rare_prob_Bellman_IS}
\end{align}
where $\rho_n(s_n, s_{n+1}, a_A)$ is the importance weight at step $n$,
\begin{align}
    \rho_n(s_n, s_{n+1}, a_A) = \frac{ \pi_A(a_A|s_n) \cdot  p_{E,gt}(s_{n+1} | a_A, s_n)}{ \pi_A^{(n)}(a_A|s_n) \cdot p_{E}^{(n)}(s_{n+1} | a_A, s_n) }. \label{eq:is_weight}
\end{align}

\textbf{a) Environment Nature as an Adversary.}
One key idea in APE is to treat uncertainties in environment transition probabilities $p_E(\cdot | \cdot, \cdot)$ (e.g., the actions of surrounding vehicles on the road or sensor noises in a robotics control task), as stochastic decisions made by the environment nature $E$ with policy $\pi_E(\cdot): \mathcal{S} \times \mathcal{A}_A \times \mathcal{A}_E \rightarrow \mathbb{R}$, where $\mathcal{A}_E$ is the environment nature's action space.
Treating the environment nature as an adversary agent is widely used in robust learning and multi-agent RL \cite{pinto2017robust}. 
It reduces the computation complexity when the selected environment agent $E$'s action space $\mathcal{A}_E$ has a smaller dimension than that of the state space $\mathcal{S}$, and relaxes the assumption of a known environment transition model $f_E: \mathcal{S} \times \mathcal{A}_A \times \mathcal{A}_E \rightarrow \mathcal{S}$. Note that we \emph{assume known and adjustable environment policy $\pi_E$ and unknown $f_E$ in \eqref{eq:env_forward}} which lead to a grey-box environment setting.

Formally, at step $n$, we have the following equations hold
\begin{align}
    p(s_{n+1} | a_{A,n}, s_n) = \pi_{E}(a_{E,n}| a_{A,n}, s_n)\\
    s_{n+1} = f_E(s_n, a_{A,n}, a_{E,n}) \label{eq:env_forward}\\
    a_{E,n} = f_E^{-1}(s_n, a_{A,n}, s_{n+1}) \label{eq:env_inverse}
\end{align}

Moreover, we only focus on importance sampling over $p_E(s_{n+1} | a_A, s_n)$ and leaves $\pi_A^{(n)}(a_A|s_n) = \pi_A(a_A|s_n), \forall n$. 
Therefore the importance weight in \eqref{eq:is_weight} becomes
\begin{align}
    \rho_n(s_n, s_{n+1}, a_{A,n}) & = \frac{ p_{E, gt}(s_{n+1} | a_{A,n}, s_n)}{ p_E^{(n)}(s_{n+1} | a_{A,n}, s_n) } \nonumber \\
    & =\frac{ \pi_{E, gt}(a_{E,n} | a_{A,n}, s_n)}{ \pi_{E}^{(n)}(a_{E,n} | a_{A,n}, s_n) }. \label{eq:ape_is_weight}
\end{align}

\textbf{b) Updating Rules for $v(s)$ and $\pi_E(a_E | s, a_A)$.}
Calculating the expectation in \eqref{eq:rare_prob_Bellman_IS}, which includes the rare event probability, is generally intractable in non-discrete settings. 
In APE, we use TD method \cite{sutton2018reinforcement} which combines MC simulation and dynamic programming to update the value $v(s)$. More concretely, at the end of step $n$,
\begin{align}
   & v^{(n+1)}(s_n)  \leftarrow  (1-\alpha_n)v^{(n)}(s_n) +   \nonumber \\
    &  \sum_{a_{A} \in \mathcal{A}_A} \rho_n (s_n, s_{n+1}, a_{A})  \alpha_n [r(s_n, s_{n+1}) + v^{(n)}(s_{n+1})] \label{eq:ape_v_update_exact} \\
    & \approx (1-\alpha_n) v^{(n)}(s_n)+ \alpha_n \cdot \nonumber\\
    & \quad \ [r(s_n, s_{n+1}) + v^{(n)}(s_{n+1})] \cdot \rho_n (s_n, s_{n+1}, a_{A,n}), \label{eq:ape_v_update_approx}
\end{align}
where $\alpha_n$ is a decaying learning rate. 
To make APE suitable for online learning and get rid of the summation over agent actions, we approximate the importance weight over state transition probability with the one-step importance weight in \eqref{eq:ape_is_weight} to get an online stochastic version as in \eqref{eq:ape_v_update_approx}.
\eqref{eq:ape_v_update_approx} is accurate if the agent policy $\pi_A$ is deterministic. 

Motivated by ASA \cite{ahamed2006adaptive}, which accelerates the rare event probability estimation in Markov chains, we design the updating rule for $\pi_E$ as follows.
At the end of step $n$, the un-normalized policy around the newly sampled action $a_{E,n}$ is derived with the updated value function $v^{(n+1)}$ and the original environment policy $\pi_{E, gt}$
\begin{align}
    & \Tilde{\pi}_E^{(n+1)}(a_{E,n}| a_{A,n}, s_n) \leftarrow \text{max} \Big( \delta , \nonumber\\ 
    &\pi_{E,gt}(a_E| a_{A,n}, s_n) \Big( \frac{r(s_n, s_{n+1}) + v^{(n+1)}(s_{n+1})}{v^{(n+1)}(s_n)} \Big)  \Big),
   \label{eq:MDP_p_update}
\end{align}
where $\delta$ is a positive constant to guarantee that the one-step importance weight $\rho_n$ is bounded. 

APE accelerates the policy value estimation by learning a properly designed importance policy $\pi_E$ to reduce the estimation variance. 
Different from adaptive MC \cite{desai2001markov} and cross-entropy methods (CEM) \cite{de2000analysis}, 
APE utilizes a stochastic approximation paradigm and updates state values based on each single transition without waiting for a final outcome (without entering the terminal set).

\textbf{c) Convergence in Tabular Case}
\begin{theorem} Assume tabular MDP settings with discrete state and action spaces and let $\pi_E^{\star}$ denote the zero-variance environment policy. If step size $\alpha_n$ in Eq.~\ref{eq:ape_v_update_exact} satisfies $\sum_n \alpha_n = \infty$ and $\sum_n \alpha_n^2 < \infty$, the iterative updating rules (Eq.~\ref{eq:ape_v_update_exact} and Eq.~\ref{eq:MDP_p_update}) satisfies $v^{(n)}\rightarrow v^{\star}$ and $\pi_E^{(n)} \rightarrow \pi_E^{\star}$. 
\end{theorem}

The proof of Theorem 1 directly follows the structure of Theorem 1 in \cite{ahamed2006adaptive}. By defining the value function only related to state, the difference between the APE value update rule and that in \cite{ahamed2006adaptive} lies in the definition of the importance weight. However, the importance weight-related terms are included in a sequence of random vectors with zero mean and bounded norm, which do not affect the contraction of the remaining part. 
Based on the Perron-Frobenius theorem, we can show that the remaining part is a contraction operator and thus has a fixed point. Then by applying Theorem 1 and Theorem 3 in \cite{tsitsiklis1994asynchronous}, we have the convergence of value function $v$. The convergence of $\pi_E$ results from the convergence of $v$ and the definition of $\pi_E$'s updating rule.

\begin{figure}[t]
    \centering
    \includegraphics[width=5.2cm]{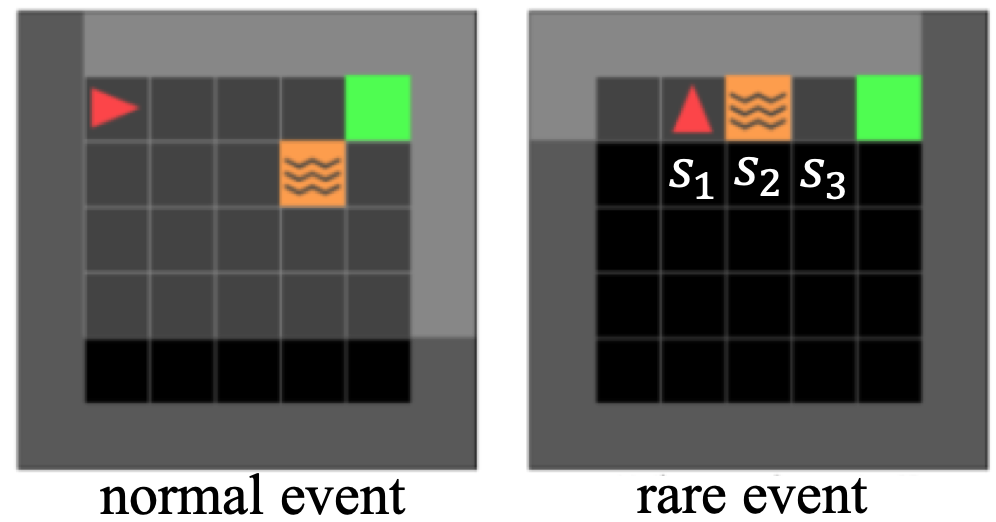} 
    \caption{A minigrid environment with rare events.
    \vspace{-0.1in}
    }
    \label{fig:minigrid}
\end{figure}

\begin{table}[t]
    \centering
    \caption{Estimation results at convergence in minigrid. We obtain the ground truth rare event probabilities via vanilla MC with a sufficiently large number of episodes.
    }
\begin{tabular}{c|c|c|c|c}
    \hline
    methods & \multicolumn{2}{|c|}{MC} & \multicolumn{2}{|c}{APE} \\ \hline
    metric  & mean & std & mean & std \\ \hline
     $v(s_1) \ (10^{-3})$ & $3.99$ & $0.38$ & $3.90$ & $0.24$ \\ \hline
     $v(s_2)\ (10^{-3})$ & $3.93$ & $0.47$ & $3.94$ & $0.23$ \\ \hline
     $v(s_3)\ (10^{-3})$ & $4.00$ & $0.50$ & $3.94$ & $0.17$ \\ \hline
     num. of time steps & $78300$ & $10600$ & $8884$ & $2595$ \\ \hline
     num. of episodes & $19306$ & $2612$ & $552$ & $134$ \\ \hline
\end{tabular}
    
    \vspace{-0.2in}
    \label{tab:grid}
\end{table}

\subsection{Accelerated Policy Evaluation with gym-minigrid}
\label{sec:ape_grid}

We first empirically show that the proposed APE method \emph{without using function approximations} converges to the ground truth rare event probabilities in discrete MDP settings. 
We experiment in a gym-minigrid \cite{gym_minigrid} environment as shown in Fig.~\ref{fig:minigrid}.  
Rare events happen when the agent (the red triangle) fails to reach its destination (the green square), such as getting stuck in non-terminal states or entering the lava (the yellow block).
Grids denoted as $s_1$, $s_2$ and $s_3$ are possible initial states of lava.
The agent policy $\pi_A$ is trained with the deep Q learning \cite{mnih2013playing} until convergence. The environment policy $\pi_E$ controls the policy of the lava.

We compare the efficiency of the proposed APE with vanilla MC for estimating the rare event probability. We define convergence of an estimated value when the estimation deviation in the past $2000$ steps is less than $1\%$ of the final estimate. 
Tab.~\ref{tab:grid} shows that when the ground truth rare event probability is around $4\times 10^{-3}$, APE successfully converges to the true rare event probability. APE also demonstrates high efficiency by requiring one magnitude fewer data and achieves a smaller variance than MC at convergence.


\section{Scalable Implementation: APE\_scalable }
\label{sec:ape_continuous}

In this section, we introduce how to improve the scalability of APE to handle large discrete or continuous MDP settings. The general idea is to select proper function approximators to represent the value function $v(s)$ and the environment policy $\pi_E$. 
Before doing that, we first establish some notations.
The value function approximator $v_{\psi}(s)$ with parameter $\psi$ takes state $s$ as input and outputs a probability ranging from 0 to 1.
The policy approximator $\pi_{E,\theta}$ parameterized with $\theta$ takes the state-action concatenation $(s, a_A)$ as input and outputs a probability density on $a_E$.
Therefore, we have $\rho_{n,\theta} = \pi_{E,gt}(a_{E, n} | a_{A,n}, s_n) / \pi_{E,\theta}^{(n)}(a_{E, n} | a_{A, n}, s_n) $.
At step $n$, a data pair $d_n  = (s_n, a_{A,n}, a_{E, n}, s_{n+1}, r(s_n, s_{n+1}), \rho_{n,\theta}) $ is appended to a history dataset $\mathcal{D}$.
Assume $s \in \mathcal{S} \subset \mathbb{R}^{d_s}$, $a_A \in \mathcal{A}_A \subset \mathbb{R}^{d_a}$, and $a_E \in \mathcal{A}_E \subset \mathbb{R}^{d_e}$.

\subsection{Value Function Approximation} 
We learn $v_{\psi}(s)$ by minimizing differences between the predicted values and TD targets $v_{\psi, TD}(s_n)$. For data pair $d_n$, we have
\begin{align}
    \rho_{n, \theta} = \pi_{E, gt}(a_{E,n} | a_{A,n}, s_n) / \pi_{E, \theta}^{(n)}(a_{E,n} | a_{A,n}, s_n), \label{eq:scalable_ape_is_weight}\\
    v_{\psi, TD}(s_n) = ( r(s_n, s_{n+1}) + v_{\psi}(s_{n+1})  )\cdot \rho_{n, \theta}. \label{eq:scalable_TD_target}
\end{align}

One approximation approach is representing $v(s)$ with an expressive deep neural network (DNN) with low prediction complexity. However, DNNs are sensitive to imbalanced data and suffer from the catastrophic forgetting problem \cite{krawczyk2016learning, chrysakis2020online}, especially in online learning settings. 
A competitive model for DNN is the Gaussian Process (GP) regression model. The equivalence between GPs and infinitely wide DNNs is derived in \cite{lee2017deep}. GPs are data-efficient and flexible in making predictions as non-parametric models but sacrifice the prediction complexity \cite{rasmussen2003gaussian}.  
Considering that the sampled rare events may be drowned out in $\mathcal{D}$, especially in the initial learning phase, we choose to use GP as the function approximator for its data efficiency. 

\textbf{a) Use GP to Represent $v_{\psi}(\cdot)$.} Let the data buffer for training and prediction as $\mathcal{D}_{GP} = \{ (s_i, v_{\psi, TD}(s_i) )^m \}$. The predicted value given $s^\star \in \mathcal{S}$ is drawn from a distribution 
\begin{align}
    p(f | \mathcal{D}, s^\star) = \mathcal{N}(\text{m}(f), \text{cov}(f)). \label{eq:GP}
\end{align}
Define $x_i = s_i$ and $y_i = v_{\psi, TD}(s_i)$ and aggregate them into $X \in \mathbb{R}^{d_s \times m }$ and $Y \in \mathbb{R}^{1 \times m}$.
The mean function is $ \text{m}(f) =K(s^\star, X)[K_i(X, X) + \sigma^2 I]^{-1}Y$ and the covariance matrix is $\text{cov}(f) = K(s^\star, s^\star)  - K(s^\star, X)[K(X, X) + \sigma^2 I]^{-1} K_i(X, s^\star)$. 
$\sigma$ is the standard deviation of observation noise.
The matrix $K$ is fully specified by the kernel function $k(\cdot, \cdot)$, which defines the function smoothness. We use the scaled squared exponential kernel $k(x, x') = w^2 \exp ( -\frac{1}{2} \sum_{j=1}^{d_s} w_{j}(x^j- x'^j)^2 )$, where $w_{j}$ is the reciprocal of the lengthscale of state dimension $j$, and $w$ is the output scale. $\psi = \{ w, w_{1}, ... , w_{d_s}, \sigma \}$ contents kernel parameters.

We update $\psi$ with gradient descent to minimize the negative log-likelihood for $\mathcal{D}_{GP}$ together with an $L_1$ regularizer over lengthscale parameters:
\begin{align}
    &\psi \leftarrow \psi- \alpha_v \nabla_{\psi} \Big[ \lambda \cdot \sum_{j}^{d_s}|w_j|   \nonumber \\
    &+ \sum_{ (s_i, v_{\psi, TD}(s_i)) \in \mathcal{D}_{GP}} -\log p(v_{\psi, TD}(s_i) | s_i, \mathcal{D}_{GP}) \Big], \label{eq:J_gp_batch_gd}
\end{align}
where $\lambda$ is the regularization penalty, and $\alpha_v$ is the learning rate. Different from standard GP regression with a fixed number of data points and fixed targets, APE learns GP in an online learning manner via streaming collected data pairs \cite{nguyen2008local} as well as in a dynamic programming paradigm with changing TD targets similar to \cite{deisenroth2008approximate}.


\subsection{Environment Adversary Policy Approximation} 
In online setting, the environment's policy updates based on the most recent data pair $d_n$ and newly updated $v_{\psi}(s)$. The un-normalized target density around the newly sampled action $a_{E,n}$ conditioned on $C_n = [a_{A,n}, s_n]$ is
\begin{align}
    \Tilde{\pi}_{E}(a_{E,n} | C_n) =  \pi_{E}(a_{E,n} | C_n) \Big( \frac{r(s_n, s_{n+1}) + v_{\psi}(s_{n+1})}{v_{\psi}(s_n)} \Big).
    \label{eq:NF_point_target}
\end{align}
To cope with continuous spaces, we achieve the modification by adding a normal distribution $\mathcal{N}(a_{E,n}, \sigma)$ with $\sigma$ much smaller than that of the ground truth policy $\pi_{E}(a_E | x, a_A )$. The target probability density function conditioned on $C_n$ is
\begin{align}
    p_{E}(a_E | C_n) = (\pi_{E,\theta}(a_E | C_n) + \beta \cdot \mathcal{N}(a_{E,n}, \sigma))/(1+\beta),
    \label{eq:NF_target}
\end{align}
where $\beta = \sqrt{2 \pi} \sigma (\Tilde{\pi}_{E}(a_{E,n} | C_n) - \pi_{E,\theta}(a_{E,n} | C_n))$.

\begin{algorithm}[tb]
\caption{Scalable Accelerated Policy Evaluation}
\label{alg:algorithmAPE}
\textbf{Input}: Agent policy to evaluate $\pi_A$, ground truth environment policy $\pi_{E, gt}$, warm start policy $\hat{\pi}_E$\\
\textbf{Parameter}: Total evaluation steps $N$  \\
\textbf{Output}: $v_{\psi}(\cdot)$, $\pi_{E,\theta}(\cdot)$
\begin{algorithmic}[1] 
\STATE Pretrain $\pi_{E,\theta}(\cdot | a_A, x)$ with target conditional probability $\hat{\pi}_{E}(\cdot | a_A, s), \ \forall a_A\in \mathcal{A}_A, s \in \mathcal{S} $
\FOR{$n=0$ {\bfseries to} $N-1$}
\STATE Reset environment
\STATE Initialize empty episode buffer $\mathcal{D}_e$
\REPEAT
\STATE Sample actions $a_A \sim \pi_A(\cdot | s)$, $a_E \sim \pi_{E, \theta}(\cdot|a_A, s)$
\STATE Execute $a_A$ and $a_E$, observe $s'$ and $r(s, s')$
\STATE Get one-step importance weight $\rho_{n, \theta}$ with \eqref{eq:scalable_ape_is_weight}
\STATE Add $d =(s, a_A, a_E, s', r(s, s'), \rho_{n, \theta})$ to $\mathcal{D}$ and $\mathcal{D}_e$
\STATE Update $v_{\psi}(\cdot)$ with two-timescale based on Algo.\ref{alg:algorithmAPEV}
\STATE Update $\pi_{E, \theta}$ with dense rewards based on Algo.\ref{alg:algorithmAPENF}
\STATE $s \leftarrow s'$
\STATE $n \leftarrow n+1$
\UNTIL{episode finish}
\ENDFOR
\STATE \textbf{return} $v_{\psi}(\cdot)$, $\pi_{E,\theta}(\cdot)$
\end{algorithmic}
\end{algorithm}

The function approximation $\pi_{E, \theta}$ should have a low sample complexity to quickly sample actions at each time step.
$\pi_{E, \theta}$ is desired to have interpolation/extrapolation ability across conditions to help accelerate the evaluation process by ensuring that two conditional distributions are close if the distance between their conditioned values is close.
Considering that the additive modification of the target density function in \eqref{eq:NF_point_target} results in a Gaussian mixture model with an increasing mixture size, $\pi_{E, \theta}$ is also desired to be flexible enough to model multi-mode distributions. 
In practice, we use the conditional Masked Autoregressive Flow model (cMAF) model \cite{papamakarios2017masked} to represent $\pi_E(a_E|a_A, s)$ to meet previous desiderata.

\textbf{a) Use cMAF to Represent $\pi_{E,\theta}$.} We obtain a closed-form parametric representation of $\pi_E(a_E|a_A, s)$ using a cMAF. A cMAF is a type conditional Normalizing Flow (cNF) \cite{papamakarios2017masked}, which is a generative model that uses invertible mappings to transform a simple probability distribution into a complex one conditioned on other random variables.
Compared with sample-based representation approaches such as MCMC, cNF directly generates one sample by calling one reversible path 
and has the capacity to model distributions that go beyond single-mode Gaussian distributions. 

We denote the generative function in cNFs as $g_{\theta}: z \rightarrow a_E$, and the differentiable normalizing function as $f_{\theta} = g^{-1}_{\theta}$.
The base distribution $p_Z$ of hidden variable $Z$ is easy to sample from, such as a normal distribution or a uniform distribution. At step $n$, we hope to minimize the KL divergence $L_{\theta}^{(n)}$ between the target and the cNF distribution.
\begin{align}
    L_{\theta}^{(n)} = & \mathbb{E}_{\pi_{E, \theta}(a_E | C_n)} [ \log \pi_{E, \theta}(a_E | C_n) - \log p_E(a_E|C_n)] \nonumber\\
    =  & \mathbb{E}_{p_{Z}(z | C_n)} [ \log \pi_{E, \theta}(g_{\theta}(z) | C_n) - \log p_E(g_{\theta}(z)|C_n)] \nonumber  \\  
    = & \mathbb{E}_{p_{Z}(z | C_n)} [\log p_{Z}(z | C_n) - \log |\det \partial_z g_{\theta}(z)| \nonumber \\
    & - \log p_E(g_{\theta}(z)|C_n)].
    \label{eq:NF_loss}
\end{align}
We update $\theta$ via gradient descent. We first sample from the simple base distribution $p_{Z}(z | C_n)$ and pass the sampled $z_i, i=[M]$ to the flow to get samples in target domain $a_{E,i} = g_{\theta}(z_i), i=[M]$. The KL divergence $L_{\theta}^{(n)}$ is then approximated with the sample mean. With learning rate $\alpha_{\pi}$,
\begin{align}
    \theta  \leftarrow & \theta - \alpha_{\pi} \nabla_{\theta} \sum_{i=1}^M  \Big[ \log p_{Z}(z_i | C_n) \nonumber \\
    &- \log |\det \partial_{z} g_{\theta}(z_i) |  - \log p_E(g_{\theta}(z_i)|C_n) \Big].
    \label{eq:NF_sgd}
\end{align}

\begin{algorithm}[tb]
\caption{Update $v_{\psi}$ with Two-timescale}
\label{alg:algorithmAPEV}
\textbf{Input}: new data $d$, current $\pi_{E, \theta}$\\
\textbf{Parameters}: $v_{\psi}$ learning rate $\alpha_v$, $\psi$ update interval $n_\psi$ \\
\textbf{Output}: $v_{\psi}$
\begin{algorithmic}[1] 
\STATE Calculate TD target $v_{\psi, TD}(s)$ of $d$ with (\ref{eq:scalable_TD_target})
\STATE Append $(s, v_{\psi, TD}(s))$ to $\mathcal{D}_{GP}$
\IF{ $n \% n_\psi = 0$ }
\STATE Train value function $v_\psi$ with gradient descent (\ref{eq:J_gp_batch_gd})
\ENDIF
\end{algorithmic}
\end{algorithm}

\begin{algorithm}[tb]
\caption{Update $\pi_{E, \theta}$ with Dense Rewards}
\label{alg:algorithmAPENF}
\textbf{Input}: new data $d$, current $v_{\psi}$, episode buffer $\mathcal{D}_e$\\
\textbf{Parameters}: $\pi_{E, \theta}$ learning rate $\alpha_{\pi}$, dense reward discount factor $\gamma_r$, momentum coefficient $\beta_\theta$ \\
\textbf{Output}: $\pi_{E,\theta}$
\begin{algorithmic}[1] 
\STATE Append $d$ to $\mathcal{D}_e$
\IF{ $r(s, s^\prime)=1$}
\STATE Calculate adversary policy target with (\ref{eq:NF_target}) based on episode buffer $\mathcal{D}_e$ and dense rewards 
\STATE Get updated $\theta^\prime$ of $\pi_{E, \theta}(\cdot)$ with gradient descent (\ref{eq:NF_sgd})
\STATE $ \theta \leftarrow \beta_\theta \theta^\prime + (1-\beta_\theta) \theta$
\ENDIF
\end{algorithmic}
\end{algorithm}

\subsection{Algorithm Description} 
We present the APE\_scalable algorithm in Algo.~\ref{alg:algorithmAPE} and summarize the algorithm highlights as follows:

\textbf{Warm start.} We initialize the environment policy $\pi_{E}$ with a warm start policy $\hat{\pi}_E$ when the rare event probability $v^\star(s_0)$ is pretty small. The design of the warm start distribution may leverage domain knowledge. One possible choice is using large variances. 

\textbf{GP two-timescale update rules.} 
We develop a two-timescale updating scheme for GP dynamic programming. We update data pairs more frequently than the GP kernel parameters to improve training stability as in Algo.~\ref{alg:algorithmAPEV}.
We store representative data as in sparse GP \cite{tran2015variational} and early-stop appending data when sufficient representative data points.

\textbf{cMAF update with dense rewards.} 
In practice, purely updating $\pi_{\theta}$ in an online manner raises several issues. First, it causes that cMAF easily over-fits to the latest target distribution across conditions (diminishing the effect of conditions). Second, it may induce training instability due to a quite sparse reward signal $r(s_{n-1}, s_{n})$. Third, it requires high computational complexity for frequently updating cMAF. Therefore, we propagate the reward signal to the whole episode $\hat{r}(s_{n-1}, s_n) = \gamma_r^{\tau-n} r(s_{\tau-1}, s_\tau), n = 1, \dots, \tau,$
where $\gamma_r$ is the reward discount factor. $\tau$ is the episode termination step. To improve training stability, we further add momentum to the parameter $\theta$. We early-stop training when achieving satisfactory sampled rare event rate.



\section{Experiments}
\label{sec:experiments}
In this section, We empirically show that the proposed APE\_scalable method \emph{with function approximations} outperforms baselines by estimating the rare event probability with smaller biases and orders of magnitude fewer samples. 

\vspace{-0.05in}
\subsection{Baselines}
\vspace{-0.05in}
We compare the proposed APE\_scalable with three baselines: \emph{APE\_discrete}, \emph{CEM\_discrete} and vanilla \emph{MC}. APE$\_$discrete and CEM$\_$discrete discretize each dimension of $s$ and $a_E$ with a small slot interval, and use dictionaries to store the rare event probability estimate and the environment policy $\pi_E$ at each slot. Note that vanilla MC does not require $\pi_E$ for estimation or store dictionaries. We get the ground truth rare event probability $v^{\star}(s_0)$ via MC with sufficiently large number of episodes. 
Let $\hat{v}(s_0)$ denote the estimation and $\bar{v}(s_0)$ denote the sampled rare-event rate along estimation process.
We evaluate the estimation error based on bias $Bias(s_0) = \mathbb{E}[\hat{v}(s_0)] - v^\star(s_0)$, standard deviation $Std(s_0) = \sqrt{\mathbb{E}[ (\mathbb{E}[\hat{v}(s_0)] - \hat{v}(s_0))^2 ]}$, and risk $MSE = \sqrt{Bias^2(s_0) + Std^2(s_0)}$.

\begin{figure}[t]
    \centering
    \includegraphics[width=\linewidth]{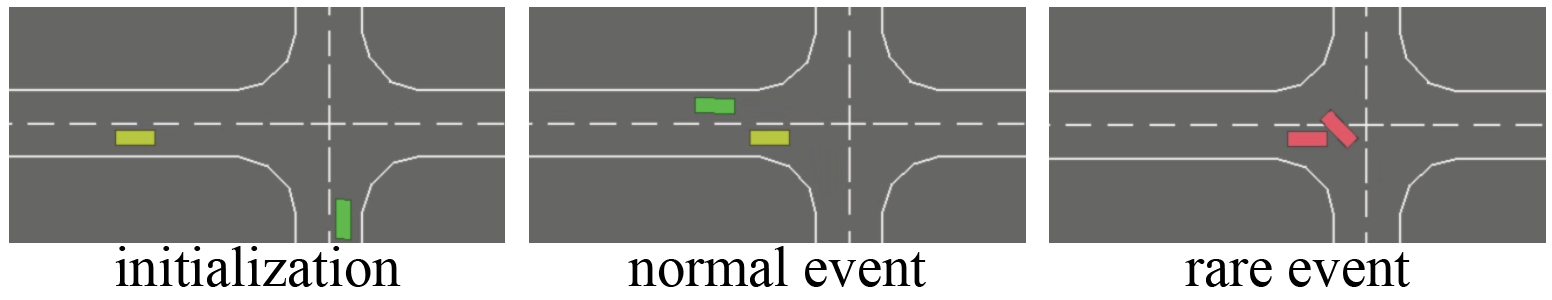} 
    \caption{The \textbf{Intersection} environment with rare events. 
    \vspace{-0.1in}
    }
    \label{fig:highway_env}
\end{figure}

\begin{figure}[t]
    \centering
    \includegraphics[width=\linewidth]{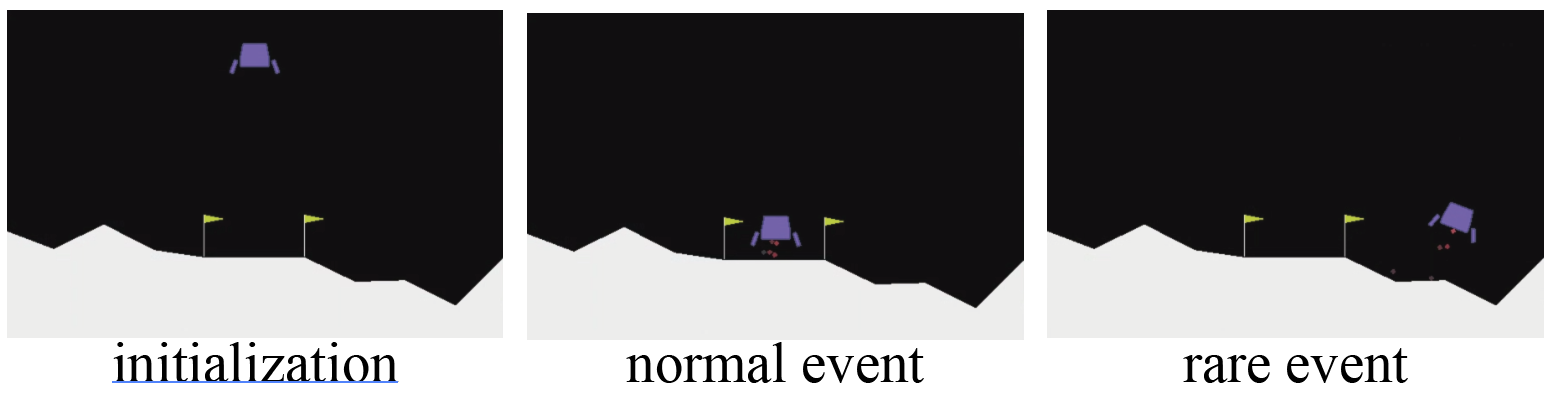} 
    \vspace{-0.1in}
    \caption{The \textbf{LunarLander} environment with rare events. 
    \vspace{-0.2in}
    }
    \label{fig:lunarlander_env}
\end{figure}

\begin{figure*}[ht]
\begin{subfigure}[]{0.333\linewidth}
  \centering
  \includegraphics[width=1.01\linewidth]{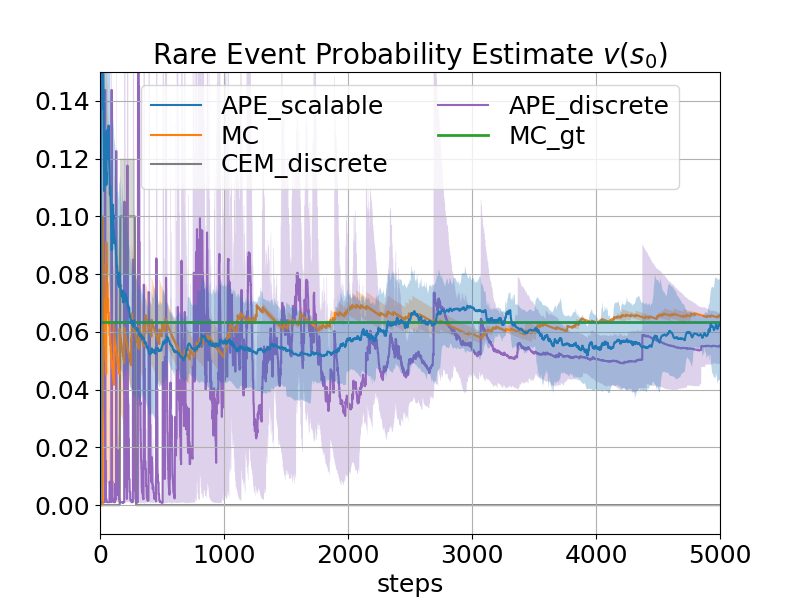}
\end{subfigure}
\hspace{-1.1em}%
\begin{subfigure}[]{0.333\linewidth}
  \centering
  \includegraphics[width=1.01\linewidth]{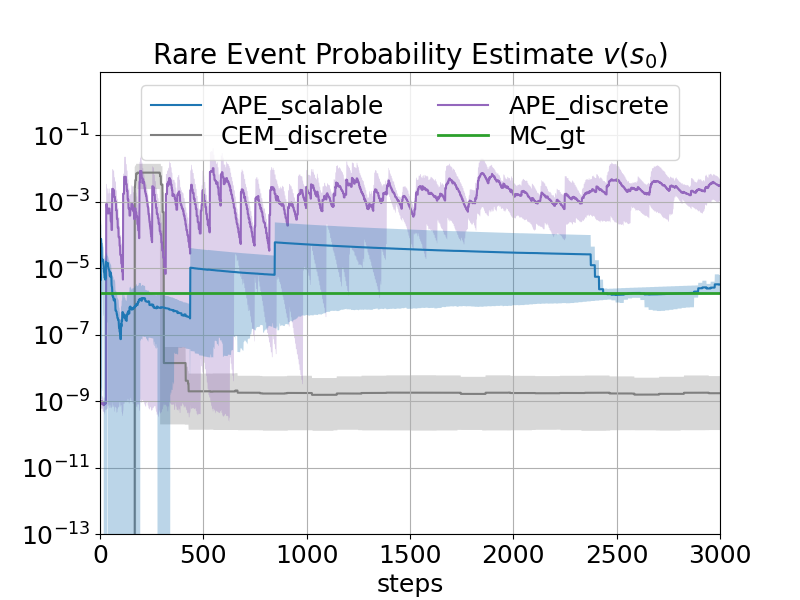}
\end{subfigure}
\hspace{-1.1em}%
\begin{subfigure}[]{0.333\linewidth}
  \centering
  \includegraphics[width=1.01\linewidth]{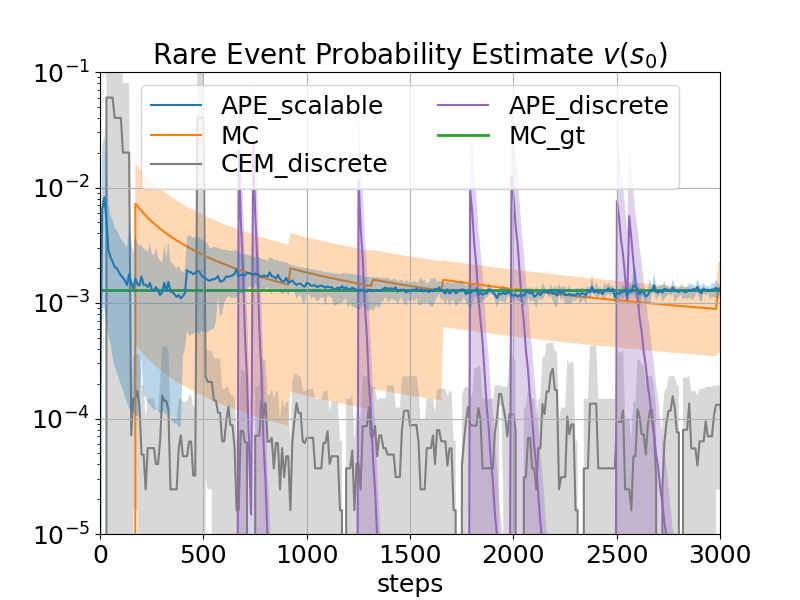}
\end{subfigure}


\begin{subfigure}[]{0.333\linewidth}
  \centering
  \includegraphics[width=1.01\linewidth]{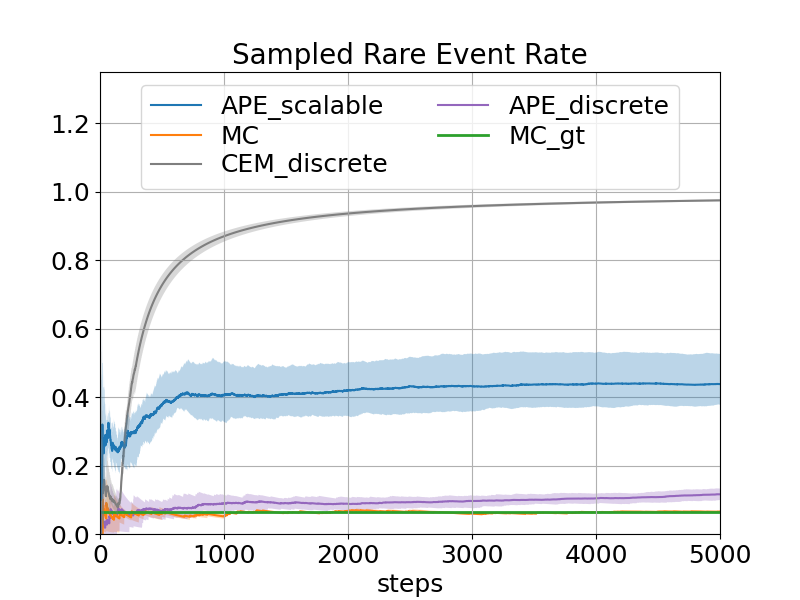}
  \caption{
  \centering $\mathtt{Int}$-$\mathtt{v0}:$
  $v^\star(s_0) = 6.34\times 10^{-2}$}
  \label{fig:intersection-v0}
\end{subfigure}
\hspace{-1.1em}%
\begin{subfigure}[]{0.333\linewidth}
  \centering
  \includegraphics[width=1.01\linewidth]{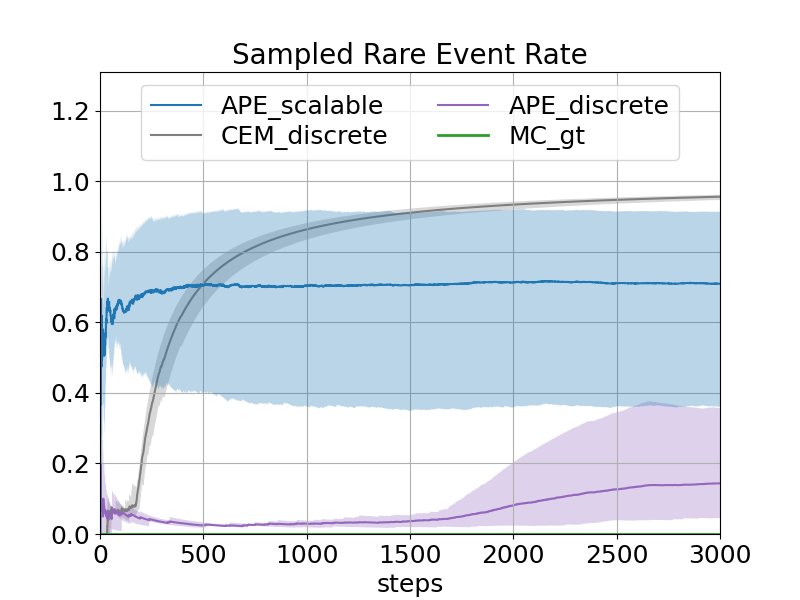}
  \caption{
  \centering $\mathtt{Int}$-$\mathtt{v1}:$ 
  $v^\star(s_0) = 1.70\times 10^{-6}$}
  \label{fig:intersection-v1}
\end{subfigure}
\hspace{-1.1em}%
\begin{subfigure}[]{0.333\linewidth}
  \centering
  \includegraphics[width=1.01\linewidth]{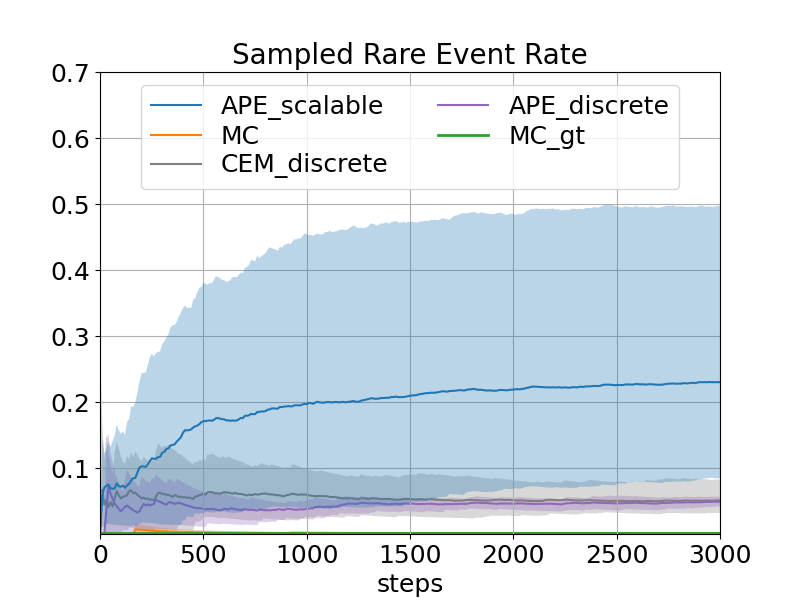}
  \caption{\centering $\mathtt{LunarLander}:$     
  $v^\star(s_0) = 1.30\times10^{-3}$}
  \label{fig:lunarlander}
\end{subfigure}
\vspace{-0.07in}
\caption{Evaluation results in $\mathtt{Int}$-$\mathtt{v0}$, $\mathtt{Int}$-$\mathtt{v1}$ and $\mathtt{LunarLander}$. Each line is run with 5 random seeds. The shaded area has lower and upper bound corresponding to 5\% and 95\% quantile. APE$\_$scalable performs well in moderate high-dimensional tasks with higher estimate accuracies and sampled rare event rates compared with baselines.
The performances of all the methods are measured with estimate errors and sampled rare event rates.
\vspace{-0.1in}
}
\label{fig:exp_plots}
\end{figure*}

\begin{table*}[t]
    \centering
    \caption{Numerical evaluation results in $\mathtt{Int}$-$\mathtt{v0}$, $\mathtt{Int}$-$\mathtt{v1}$ and $\mathtt{LunarLander}$. Each number is evaluated with 5 random seeds. The best performance (the smallest number in each column) is highlighted in bold.
    }
    \begin{tabular}{c|ccc|ccc|ccc}
         Environment & \multicolumn{3}{c|}{Int-v0 (Unit $10^{-2}$)}  & \multicolumn{3}{c|}{Int-v1 (Unit $10^{-6}$)} & \multicolumn{3}{c}{LunarLander (Unit $10^{-3}$)} \\ \hline
         Metrics & $Bias$ & $Std$ & $MSE$ & $Bias$ & $Std$ & $MSE$ & $Bias$ & $Std$ & $MSE$ \\ \hline
        APE\_scalable & \textbf{-0.07} & 6.27 & 6.27 & \textbf{1.08} & 2.88 & 3.07 & \textbf{-0.02} & 1.28 & 1.28 \\\hline
        APE\_discrete & -0.75 & 0.47 & \textbf{0.88} &  2447.75 & 149.44 & 2452.31 & -1.30 & \textbf{0.00} & 1.16\\\hline
        CEM\_discrete & -6.34 & \textbf{0.00} & 6.34 & -1.70 & \textbf{0.00} & \textbf{1.70} &  -1.17 & 0.07 & 1.27 \\\hline
        MC & 0.23 & 6.57 & 6.57 & -1.70 & \textbf{0.00} & \textbf{1.70} & 0.08 & 0.42 & 0.43 \\\hline
    \end{tabular}
    \vspace{-0.15in}
    \label{tab:exp_results}
\end{table*}

\begin{figure}[]
\centering
\includegraphics[width=0.9\linewidth]{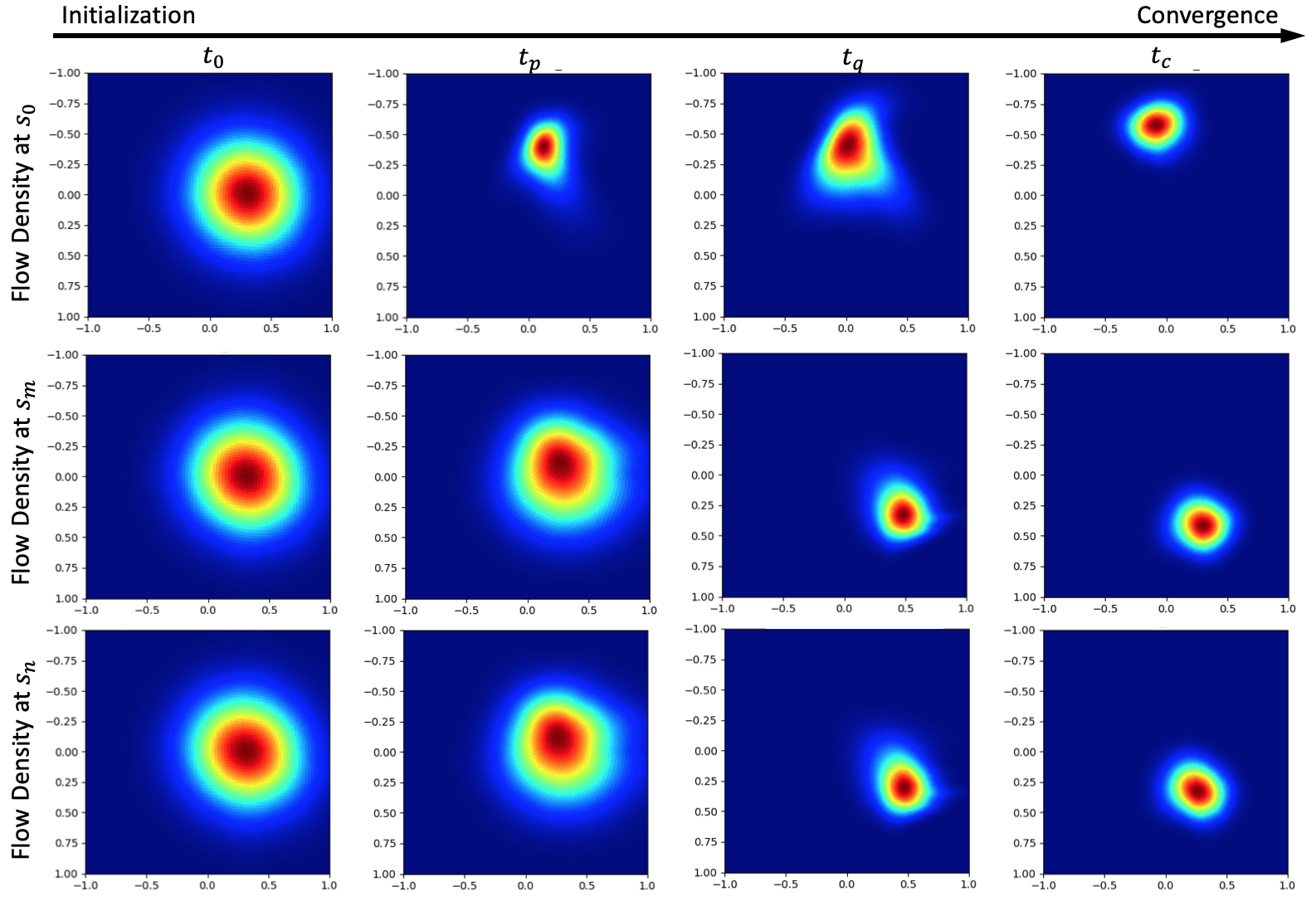}
\caption{The evolution of $\pi_E$ in $\mathtt{Int}$-$\mathtt{v1}$ conditioned at the initial state $s_0$, a random intermediate state $s_m$ and a state before collision $s_n$. The time steps $t_0 < t_p < t_q < t_c$.
\vspace{-0.2in}
}
\label{fig:evolution_pi_E_intersection}
\end{figure}

\vspace{-0.02in}

\subsection{Environments}
\vspace{-0.05in}
\textbf{Intersection}
The agent to evaluate (the green vehicle) aims to perform a collision-free unprotected left turn in a two-lane intersection with one surrounding vehicle (the yellow vehicle). A rare event is a crash, and the adversary policy is the surrounding vehicle's policy. We hope to evaluate a fixed deterministic policy $\pi_A$, which follows the lane with constant velocity. Note that APE is also applicable to stochastic policies.
We build two intersection scenarios based on highway-env \cite{highway-env}: $\mathtt{Int}$-$\mathtt{v0}$ with $v^\star(s_0) =6.34\times 10^{-2}$, and $\mathtt{Int}$-$\mathtt{v1}$ with $v^\star(s_0) =1.70 \times 10^{-6}$.
APE can handle more than two agents by treating all the other agents as part of the environment \cite{albrecht2018autonomous}.
In $\mathtt{Int}$-$\mathtt{v0}$, we initialize all methods with the $\pi_{E, gt}$ and present experiment results in Fig.~\ref{fig:exp_plots} (a). In $\mathtt{Int}$-$\mathtt{v1}$, we initialize APE\_scalable, APE\_discrete and CEM\_discrete with a slightly skewed distribution (as shown in the first column in Fig.~\ref{fig:evolution_pi_E_intersection}) with a larger variance than the ground truth $\pi_{E, gt}$ to sample rare events at the initial stage. We present experiment results of intersections in Fig.~\ref{fig:exp_plots} (b).

\textbf{Lunarlander}
We hope to evaluate a heuristic policy by OpenAI which aims to land a lunar lander on a fixed landing pad in $\mathtt{LunarLander}$ \cite{brockman2016openai} with $v^\star(s_0) = 1.30\times10^{-3}$.
The rare event is defined as if the lander crashes, flies too far away from the landing pad or the episode exceeds the maximum length. 
The environment agent is an adversary that controls perturbations over actions to execute to mimic controller errors.
This example provides a clue for selecting factors controlled by the environment adversary, and the experimenters may draw inspiration from robust learning and add uncertainty/noise to components of MDPs \cite{tessler2019action, zhang2020robust2, xu2012distributionally}.
In $\mathtt{LunarLander}$, we initialize all methods with the $\pi_{E, gt}$ and present experiment results in Fig.~\ref{fig:exp_plots} (c).

\vspace{-0.02in}

\subsection{Discussion about Evaluation Performance}
\vspace{-0.05in}
Although APE$\_$discrete approaches the true rare event probability in $\mathtt{Int}$-$\mathtt{v0}$, it has a larger estimate variance than APE$\_$scalable. In $\mathtt{Int}$-$\mathtt{v1}$ and $\mathtt{LunarLander}$, APE$\_$scalable successfully converges to the true value while APE$\_$discrete does not. 
The smaller variance and faster convergence of APE$\_$scalable compared with APE$\_$discrete across environments indicate that the discretization of continuous tasks drops the environment or action structure information and leads to biased evaluation results. 

Interestingly, CEM$\_$discrete achieves higher rare event sample rates than APE$\_$scalable in $\mathtt{Int}$-$\mathtt{v0}$ and $\mathtt{Int}$-$\mathtt{v1}$, but heavily underestimates the true rare event probability. 
We conjecture that CEM$\_$discrete fails to explore enough and its environment adversary policy $\pi_E$ only concentrates on a subset of failure modes. While in $\mathtt{LunarLander}$, CEM\_discrete increases the sampled rare event rate but is less efficient compared with APE\_scalable.

Both APE$\_$scalable and vanilla MC approach the true value in $\mathtt{Int}$-$\mathtt{v0}$ and $\mathtt{LunarLander}$, since the rare event probability is relatively large. Despite of that, APE still outperforms MC in terms of the variance of probability estimation in $\mathtt{LunarLander}$. But in $\mathtt{Int}$-$\mathtt{v1}$, APE$\_$scalable converges within 3,000 steps (around 1,000 episodes) while MC may require $1,762,193 = \log(0.05)/\log(1-v^\star(s_0))$ episodes until having 95$\%$ confidence interval. 
Note that we omit the evaluation curve using MC in $\mathtt{Int}$-$\mathtt{v1}$ since MC hardly samples even one rare event within 3,000 steps.
It shows that APE$\_$scalable dominates MC in estimation accuracy when the rare event probability is small.

\vspace{-0.05in}
\subsection{Discussion about Different Metrics}
\vspace{-0.05in}
From Tab.~\ref{tab:exp_results}, APE\_scalable has the smallest bias across three environments compared to baselines but not the smallest variances or MSEs. Note that although in $\mathtt{Int}$-$\mathtt{v0}$ and  $\mathtt{Int}$-$\mathtt{v1}$ CEM\_discrete has near-zero variance, it heavily under-estimates the rare event probability as shown in Fig.~\ref{fig:exp_plots}. 
Similarly, in $\mathtt{LunarLander}$, APE\_discrete underestimates the rare event probability by one magnitude but has a smaller variance than other baselines. 
Such under-estimation phenomena are risky and not desirable in the rare event setting since they may cause catastrophic consequences by deploying overly trusted RL agents.

\vspace{-0.05in}
\subsection{Discussion about Environment Adversary $\pi_E$}
\vspace{-0.05in}
In $\mathtt{Int}$-$\mathtt{v0}$, the sampled rare event rate using APE$\_$scalable (0.44) is 7.5 times greater than MC (0.064), and is 3.3 times greater than APE$\_$discrete (0.12). In $\mathtt{LunarLander}$, the final sampled rare event rate using APE\_scalable is more than 100 times greater than MC. 
Interestingly, in the more challenging $\mathtt{Int}$-$\mathtt{v1}$, APE$\_$scalable generates much more rare events with sampled rare event rate 0.75, which is 441,176 times greater than MC ($1.70 \times 10^{-6}$) and 4.7 times greater than APE$\_$discrete (0.16).
Together with the smaller estimation biases discussed before, the learnt policy $\pi_E$ with APE$\_$scalable is closer to the zero-variance distribution, and APE$\_$scalable performs better with smaller $v^\star(s_0)$.

We notice that the environment adversary policies at convergence with different random seeds hold a similar interpretable behavior mode, one example of which is visualized in Fig.~\ref{fig:evolution_pi_E_intersection} for $\mathtt{Int}$-$\mathtt{v1}$. We see that the learned adversary policy $\pi_E$ tends to accelerate (being aggressive) at the beginning of each episode and keep the original behavior $\pi_{E, gt}$ (being conservative) in the remaining steps.

\vspace{-0.02in}

\section{Conclusion}
\vspace{-0.05in}
Motivated by the limitations of existing policy evaluation methods facing rare events, we propose APE, which scales to complex tasks by drawing from powerful function approximators and explores rare events caused by sequential interactions via learning an adversary with adaptive importance sampling. We demonstrate the effectiveness of APE with its orders of magnitude sample savings to achieve convergence. APE provides a fundamental tool to allow the evaluation, and subsequent deployment, of intelligent agents in safety-critical systems that are otherwise computationally prohibitive.
It also increases policy interpretability by uncovering agent behaviors in extreme cases (rare events). 


\vspace{-0.02in}

The scalability of APE in evaluating rare events in sequential environments opens up two important future directions. First is the training of rare-event-aware intelligent agents. This requires a minimax training mechanism to optimize the agent’s strategies against the avoidance of rare events, the latter captured via the environment adversary learned from adaptive importance sampling as in this paper. Second is more powerful update methods for the value function and importance distribution, such as incorporating (Bayesian) uncertainty estimates in the GP regression model and using multi-step TD.







\section*{ACKNOWLEDGMENT}

The authors gratefully acknowledge the support from the National Science Foundation (under grants IIS-1849280 and IIS-1849304), and unrestricted research grant from the Toyota Motor North America. The ideas, opinions and conclusions presented in this paper are solely those of the authors.



\bibliographystyle{IEEEtran} 
\bibliography{ape}

\end{document}


\maketitle
\thispagestyle{empty}
\pagestyle{empty}


\section{Environments}
\label{sec:appendix_env}
We study three practical examples to show the strength and generality of APE including the discrete gym-minigrid \cite{gym_minigrid}, highway-env \cite{highway-env} which is a multi-agent continuous environment and lunar-lander \cite{brockman2016openai} which aims to solve a single-agent control task.
We assume that the evaluation policy $\pi_A$ is mature as in most rare event literature so that rare events are naturally related to failure cases similar to most existing literature. 
There are other environments used in rare-event related literature, such as MuJoCo humanoid \cite{todorov2012mujoco} (rare events as falling down situations \cite{uesato2018rigorous}) and mountain-car \cite{brockman2016openai} (rare events as small perturbations to the initial velocity \cite{ciosek2017offer}). 
The proposed Accelerated Policy Evaluation (APE) method relies on experimenters to choose factors controlled by the environment agent.

\subsection{Single-agent environment: lunar-lander}

The meanings of 8-dimensional observation and 2-dimensional action in $\mathtt{LunarLander}$ are shown in Tab. \ref{tab: obs_act_lunarlander}. The agent controls the main engine and orientation engine.
The environment adversary adds perturbations into the agent action with frequency 1 FPS. The agent controls the lander at 50 FPS. During every 50-step period, a constant action noise is injected into the commended action $a_A$. 
In the original environment, the action noise is sampled i.i.d. from a 2-D Gaussian distribution centered at $0$ with a standard deviation $0.1$. 

\begin{table}[h]
  \caption{Observation and action space of $\mathtt{LunarLander}$}
  \label{tab: obs_act_lunarlander}
  \centering
  \begin{tabular}{lll}
    \toprule
    Num     & Observation     & min, max \\
    \midrule
    0     & horizontal coordinate  & $-1.5, 1.5$     \\
    1     & vertical coordinate & $-0.5, 2.0$      \\
    2     & horizontal speed       & $-1.5, 1.5$  \\
    3     & vertical speed       & $-1.5, 1.0$  \\
    4     & angle       & $-\pi, \pi$  \\
    5     & angular speed       & $-3.0, 3.0$  \\
    6     & if first leg has contact       & $\{0, 1\}$  \\
    7     & if second leg has contact       & $\{0, 1\}$  \\
    \midrule
    Num     & Action     & min, max \\
    \midrule
    0     & main engine power  & $-1.0, 1.0$     \\
    1     & orientation enginer power & $-1.0, 1.0$      \\

    \bottomrule
  \end{tabular}
\end{table}

\subsection{Multi-agent environment: highway-gym}

The $\mathtt{Int}$-$\mathtt{v0}$ and $\mathtt{Int}$-$\mathtt{v1}$ environments are modified based on highway-env~\cite{highway-env}. We consider a two-agent example where the environment agent controls the challenger vehicle's acceleration. 
The environment action $a_E$ is the coefficient of a linear controller for controlling the challenger vehicle's acceleration. $acc_E = a_E^T x_E$ where $x_E \in \mathbb{R}^{d_{x_E}}$ is the feature vector of interest.
More concretely, the action $a_E \sim \pi_E$ is a multi-dimensional vector with each element corresponding to the coefficient of a feature. 
Denote the positions and heading of a vehicle as $(x,y,h)$.
In $\mathtt{Int}$-$\mathtt{v0}$, the state is
\begin{align*}
s = (&x_{A}, y_{A}, \dot{x}_{A}, \dot{y}_{A}, \cos{h}_{A}, \sin{h}_{A}, \\
&x_{rel}, y_{rel}, \dot{x}_{rel}, \dot{y}_{rel}, \cos{h}_{rel}, \sin{h}_{rel}) \in \mathbb{R}^{12},
\end{align*}

which includes the kinematic information of the agent and the challenger vehicle evaluated in the world frame and in the body fixed frame attached to the agent. 
The state $s$ is within $[-1, 1]$. 
The environment adversary has a action $a_E \in \mathbb{R}^3$ with value in range $[-2,2]$.
In $\mathtt{Int}$-$\mathtt{v1}$, we simply the environment with states $s = (x_{rel}, y_{rel}) \in \mathbb{R}^2$ and adversary actions $a_E \in \mathbb{R}^2$.

\section{Baselines}
\label{sec:appendix_baselines}

We use a desktop PC with twelve 64-bit CPU (model: Intel(R) Core(TM) i7-8700K CPU @ 3.70GHz) and a GPU (model: NVIDIA GeForce RTX 2080 Ti) as the computing infrastructure. 
We implement cMAF based on the paper \cite{papamakarios2017masked} as well as the public package 
\cite{kamenbliznashki2020} and 
GP based on GPytorch \cite{gardner2018gpytorch}.
When selecting model hyperparameters for APE and baselines, we first randomly search in a coarse range and then do a grid search in a smaller hyperparameter space.
APE$\_$discrete and CEM$\_$discrete both divide each dimension of $s$ and $a_E$ with small enough slot interval, and use dictionaries to store the rare event probability estimate and the environment policy $\pi_E$.

\paragraph{Cross-entropy method (CEM) with discretization} We adopt another baseline, the cross-entropy method for Markov Chains \cite{de2000analysis, de2000adaptive} to our setting by treating the environment agent's policy $\pi_E$ as the importance distribution. CEM requires to sample $k$ episodes before updating the importance distribution $\pi_E^{(m)}$. 
At update iteration $m$, the rare event probability estimate and the updating rule for $\pi_E$ in discrete settings are as follows: 
\begin{align}
    v^{(m+1)} = \frac{1}{k}\sum_{j=1}^k L_j \cdot I_j(\tau),
\end{align}
\begin{align}
    \tilde{\pi}_E^{(m+1)}(a_{E}| a_{A}, s) & \leftarrow \text{max} \Big(  \delta_{CEM} , \nonumber \\ & \frac{\sum_{j=1}^k L_j \cdot N_{j}(s, a_A, a_E) \cdot I_j(\tau)}{\sum_{j=1}^k L_i \cdot N_{j}(s, a_A) \cdot I_j(\tau)}   \Big),
\end{align}
\begin{align}
L_j  &= \prod_{n=1}^{\tau} \rho_n(s_n, s_{n+1}, a_{A,n}) \nonumber \\
& = \prod_{n=1}^{\tau} \frac{ \pi_{E, gt}(a_{E,n} | a_{A,n}, s_n)}{ \pi_{E}^{(m)}(a_{E,n} | a_{A,n}, s_n) },
\end{align}
where $\tilde{\pi}_E^{(m+1)}(a_{E}| a_{A}, s)$ is the unnormalized target. $L_j$ is the likelihood of the episode $j$.  $\delta_{CEM}$ ensures the policy has positive values. $N_{j}(s, a_A, a_E)$ is the count of the environment selecting $a_E$ given condition $(s,a_A)$ along episode $j$. $N_{j}(s, a_A)$ is the count of the environment encountering condition $(s,a_A)$ along episode $j$.

\begin{figure}[h]
\begin{subfigure}[]{\columnwidth}
  \centering
  \includegraphics[width=0.81\linewidth]{figures/update_7114_vis_step_400.png}
\end{subfigure}
\newline

\begin{subfigure}[]{\columnwidth}
  \centering
  \includegraphics[width=0.81\linewidth]{figures/update_147_vis_step_400.png}
\end{subfigure}

\caption{Visualization of $\pi_E$ in the middle of the evaluation process. The distributions in each row have the same condition $(s, a_A)$. From left to right, the figures in each row present the ground truth distribution, the target distribution, the cMAF distribution after training and cMAF samples.
}
\label{fig:multimodal_cMAF}
\end{figure}

\section{Additional experiments}

\subsection{Multi-modality of adversary policy along evaluation}

In Fig.~\ref{fig:multimodal_cMAF}, we visualize the environment agent policy $\pi_E$ at two steps in $\mathtt{Intersection}$-$\mathtt{v0}$ to show that (1) the target density at each update step is composed of an increasing number of Gaussian distributions and (2) cMAF has the capacity to represent multi-modal distributions.

\section{Convergence analysis}

\label{sec:appendix_convergence_analysis}

\subsection{Convergence analysis for APE}

\begin{theorem} Assume tabular MDP settings with discrete state and action spaces and let $\pi_E^{\star}$ denote the zero-variance environment policy. If step size $\alpha_n$ satisfies $\sum_n \alpha_n = \infty$ and $\sum_n \alpha_n^2 < \infty$, the iterative updating rules of APE satisfies $v^{(n)}\rightarrow v^{\star}$ and $\pi_E^{(n)} \rightarrow \pi_E^{\star}$. 
\end{theorem}

\begin{proof}
We state the proof sketch as follows.

The proof of Theorem 1 directly follows the structure of Theorem 1 in \cite{ahamed2006adaptive}. Note that by defining the value function only related to state, the difference between APE value update rule and that in \cite{ahamed2006adaptive} lies in the definition of the importance weight. However, the importance weight-related terms in proof are included in a sequence of random vectors with zero mean and bounded norm and do not affect the contraction of the remaining part. 
Basedon Perron-Frobenius theorem, we can show that the remaining part is a contraction operator and thus has a fixed point. Then by applying Theorem 1 and Theorem 3 of \cite{tsitsiklis1994asynchronous}, we have the convergence of value function $v$. The convergence of $\pi_E$ results in the convergence of $v$ and the definition of $\pi_E$'s updating rule.



\end{proof}

\subsection{Convergence analysis for APE\_scalable}
Recall that the value function update rule is 
\begin{align}
   & v^{(n+1)}(s_n)  \leftarrow  (1-\alpha_n)v^{(n)}(s_n) +   \nonumber \\
    &  \sum_{a_{A} \in \mathcal{A}_A} \rho_n (s_n, s_{n+1}, a_{A})  \alpha_n [r(s_n, s_{n+1}) + v^{(n)}(s_{n+1})] \label{eq:ape_v_update_exact} \\
    & \approx (1-\alpha_n) v^{(n)}(s_n)+ \alpha_n \cdot \nonumber\\
    & \quad \ [r(s_n, s_{n+1}) + v^{(n)}(s_{n+1})] \cdot \rho_n (s_n, s_{n+1}, a_{A,n}), \label{eq:ape_v_update_approx}
\end{align}

where the importance weight is
\begin{align}
    \rho_n(s_n, s_{n+1}, a_{A,n}) & = \frac{ p_{E, gt}(s_{n+1} | a_{A,n}, s_n)}{ p_E^{(n)}(s_{n+1} | a_{A,n}, s_n) } \nonumber \\
    & =\frac{ \pi_{E, gt}(a_{E,n} | a_{A,n}, s_n)}{ \pi_{E}^{(n)}(a_{E,n} | a_{A,n}, s_n) }. \label{eq:ape_is_weight}
\end{align}

In this section, we examine the convergence properties/issues associated with APE. Since APE focuses on evaluate rare event probability under a fixed policy $\pi_A$ as in  \eqref{eq:rare_prob_definition}, we can view it as an expected cost evaluation problem for a Markov chain with a fixed transition probability. 

\begin{align}
    & v^\star(s_0) \coloneqq  \mathbb{E}_{p_{E, gt}, \pi_A} \Big[ \sum_{n=0}^{\tau-1} r(s_{n}, s_{n+1}) \Big] = \label{eq:rare_prob_definition}
    \\ &\sum_{a_A \in \mathcal{A}_A} \pi_A(a_A|s_0) \sum_{s_1 \in \mathcal{S} }p_{E,gt}(s_1 | s_0, a_A) \Big[ r(s_0, s_1) + v^\star(s_1) \Big],
    \nonumber
\end{align}

The shifting policy $\pi_E$ corresponds to changing the transition probability to accelerate the evaluation, particularly for rare events under the original transition probability. To be specific, we assume a Markov chain with state $\mathcal S$, transition kernel $(p(x,y), x,y\in \mathcal S)$ and definitions  on $\mathcal I, \mathcal T, \mathcal R, \tau, (r(x,y), x,y\in\mathcal S)$ as in the case for MDP in Section 3. We can view the problem as evaluating 
\begin{equation*}
    v^\star(x)=\mathbb E_{x,p}[\sum_{n=1}^{\tau} r(s_{n-1},s_{n})]
\end{equation*}
for $s_0\in\mathcal I$, while $v^\star$ is set to 0 on $\mathcal T$. 

Formally, in the case that $|\mathcal S|<\infty$, the update rule associated with \eqref{eq:ape_is_weight} and \eqref{eq:ape_v_update_approx}, in the Markov chain setting, is analogous to updating rules in \cite{ahamed2006adaptive}:
\begin{align*}
    &v^{(n+1)}(s_n)=(1-\alpha) v^{(n)}(s_n) + \alpha  \frac{p(x_n,x_{n+1})}{p^{(n)}(x_n,x_{n+1})}\cdot  \nonumber \\
    & \quad [r(s_n, s_{n+1}) + v^{(n)}(s_{n+1})]  \quad  \text{(value iteration)}\nonumber\\
    &\tilde{p}^{(n+1)}(x_n,x_{n+1})=\max\bigg(\delta , \ p(x_n,x_{n+1}) \cdot \\
    & \quad \Big(\frac{r(s_n,s_{n+1})+v^{(n+1)}(x_{n+1})}{v^{(n+1)}(s_n)}\Big)\bigg) \quad \text{(weight update)}\nonumber\\
    &{p}^{(n+1)}(x_n,x_{n+1})= \frac{\tilde{p}^{(n+1)}(s_n,y)}{\sum_{z\in \mathcal S}
    \tilde{p}^{(n+1)}(s_n,z)} \quad \text{(normalization)}
\end{align*}
where $\delta,\alpha>0$, $v^{n}$, $p^{n}$ are the $n$-th iteration of value evaluation and adaptive importance sampling used for simulation, respectively. The convergence (with probability 1) property of the above iterative algorithm is established in \cite{ahamed2006adaptive}. However, we want to investigate the convergence behavior of \eqref{eq:ape_is_weight} and \eqref{eq:ape_v_update_approx} in a Markov chain setting where the state space $\mathcal S$ is continuous/large, and we are updating parametric/non-parametric $v_{\psi^{(n)}}$ and $p_{\theta^{(n)}}$. For this analysis, we focus on the  parameterized case (i.e. function approximator) which has better analytical properties, for insights into our practical implementation.

Before we proceed, we briefly outline the main steps and some key observations in our analysis. First, for the adaptive importance sampling $p^{(n)}_\theta$, the weight update and normalization scaling for $p^{(n)}$ in the above equations are replaced by an update on $\theta$, which is based on approximate loss minimization between distributions, e.g. KL divergence. However, even though the rate (and practical performance) of convergence towards $v^\star$ depends on how well we could update $p_{\theta^{(n)}}$ towards $p^\star$, the so-called "zero-variance" sampling distribution, whether the convergence to a fixed point (if there is one) happens with probability 1 is not affected by the evolution of $p_{\theta^{(n)}}$, as in principle one can always not use adaptive sampling to accelerate. Thus, the convergence of $p_{\theta^{(n)}}\rightarrow p^\star$ is not required for the convergence of $v_{\psi^{(n)}}$. Second, the analysis on convergence is generally not straightforward for the case $|\mathcal S|=\infty$. To handle the difficulties arising from the infinite cardinality, we discretize $\mathcal S$ to allow for a more compact representation. Consequently, we require certain regularity conditions on $(p(x,y), x,y\in\mathcal S)$ and $\{p_\theta (\cdot,\cdot):\mathcal S\times\mathcal S\rightarrow\mathbb R^{+}\}_{\theta}$. Third, the analysis relies on techniques from Approximate Value Iteration (AVI) (see \cite{ramaswamy2017analyzing}), which incorporate the error/bias in the stochastic iterative algorithms. In our analysis, such error can arise from the discretization of $\mathcal S$ and the expressiveness/accuracy of $\{v_\psi (\cdot):\mathcal S\rightarrow\mathbb R\}_{\psi}$ in approximating $(v^\star(x), x\in\mathcal S)$. 

\subsection{Definitions and assumptions}
Let $\mathcal F_n$ be the filtration generated by first n observation of the chain $\{s_0,s_1,...,s_n\}$. Let $\{p_{\theta^{(n)}}\}_n$ be the sequence of sampling distribution (derived from $\pi_E$ updates in Algorithm 1, ideally $p_{\theta^{(n)}}\rightarrow p^\star$, the zero-variance sampling distribution for estimating $v^\star$) where $\theta^{(n)}$ is $\mathcal F_n-$ measurable.

\begin{definition}\label{df1}
For a transition kernel $(q(x,y), x,y \in \mathcal S)$ and $\epsilon$, define 
\begin{equation*}
    M(\epsilon,q)\coloneqq \sup_{\|x-y\|_\infty\leq \epsilon} \frac{|q(x,z)-q(y,z)|}{q(y,z)},
\end{equation*}
with the convention $0/0=0$.
\end{definition}

\begin{assumption}\label{as1}
We assume (A1.) The state space is a compact subset of $\mathbb R^d$; (A2.) The Markov chain under the original transition kernel $p(\cdot,\cdot)$ and $\{p_{\theta^{(n)}}\}_n$ are bounded above by some constant $K$ and guarantee $\mathbb E\tau < \infty$ (which implies $\tau<\infty$ w.p.1); (A3.) There exists $\epsilon_0>0$ such that
\begin{equation*}
    \sup_{\substack{\epsilon\leq\epsilon_0 \\ q\in\{p_{\theta^{(n)}}\}_n\cup\{p\}}} \frac{M(\epsilon,q)}{\epsilon} < M
\end{equation*}
for some $M<\infty$; (A4.) There exist some constant $C>1$ such that 
\begin{equation*}
    |r(x,z)-r(y,z)|\leq C\|x-y\|_\infty,
\end{equation*}
for all $x,y,z\in\mathcal S$.

\end{assumption}

Notice the main goal of adaptive importance sampling is to approach the "zero-variance" sampling distribution for $v^\star$ which would ideally make the iteration towards "rare events" much faster. On the other hand, from a technical standpoint, the condition $\mathbb E[\tau]<\infty$ also alleviates the analytical difficulty arising from the fact that our MDP is undiscounted. We start with the following lemma.
\begin{lemma} \label{discount}
Let $\{s_n\}_{n\geq 0}$ be the Markov chain and let the non-negative reward function $r(\cdot,\cdot)$ be uniformly bounded as $\|r\|_\infty<R$. Let $\gamma\in (0,1)$ be a discount factor. Then, as $\gamma\rightarrow 1$
\begin{equation*}
    \mathbb E[\sum_{t=1}^{\tau}\gamma^{t-1}r(s_{t-1},s_t)] \rightarrow \mathbb E[\sum_{t=1}^{\tau}r(s_{t-1},s_t)],
\end{equation*}
when $\mathbb E[\tau]<\infty$ (the convergence is uniform for all $s_0\in\mathcal S$).
\end{lemma}
\begin{proof}
Since $t>0$ and it discrete, it follows from Fubini's theorem that $\mathbb E[\tau]=\sum_{t=0}^\infty P(\tau>t)$. Thus, $ \lim_{t \to \infty} \sum_{t'>t}P(\tau>t') = 0$. Given any $\epsilon>0$, we can find $T>1$ large enough such that $\sum_{t> T-1}P(\tau\geq t-1)<\frac{\epsilon}{2R}$ and $\gamma$ close enough to 1 such that $\sum_{t=1}^{T-1}(1-\gamma^{t-1})<\frac{\epsilon}{2R}$. Then,
\begin{align*}
     &\mathbb E[\sum_{t=1}^{\tau}r(s_{t-1},s_t)]-\mathbb E[\sum_{t=1}^{\tau}\gamma^{t-1}r(s_{t-1},s_t)] \nonumber\\
     =& \mathbb E[\sum_{t=1}^{\infty}(1-\gamma^{t-1})r(s_{t-1},s_t) \mathbb{I}_{ \{\tau\geq t-1 \}}] \nonumber\\
     \leq & R \sum_{t=1}^{T-1}(1-\gamma^{t-1})+R\sum_{t>T-1} P(\tau\geq t-1) <\epsilon.
\end{align*}
\end{proof}
 
As the proof indicates, the rate of convergence for the discounted expected towards the undiscounted one relies on how fast the tail of $\tau$ diminishes, which suggests how a consistent and efficient adaptive sampling distribution that can quickly realize the hitting time $\tau$ can give nice analytical properties.

\subsection{Chain discretization}
For the analysis of the continuous Markov chain derived from the MDP update \eqref{eq:ape_is_weight} and \eqref{eq:ape_v_update_approx}, we first discretize the continuous state $|\mathcal S|$. Since $\mathcal S$ is compact, for any $h>0$, we can find a finite collection of sets $\{\mathcal S^h_i\}_{i=1}^{N_h}\subseteq\mathcal S$ such that $\cup_{i\in[N_h]}\mathcal S_i^h=\mathcal S$, $\mathcal S_i^h \cap \mathcal S_j^h=\emptyset$ for $i\neq j$ and \begin{equation*}
    \max_{i\in[N_h]} \sup_{ x,y\in\mathcal S_i^h} \|x-y\|_\infty\leq h.
\end{equation*}
To simplify analysis, we further assume that there are no overlaps between different types of states among $\{\mathcal S^h_i\}_{i=1}^{N_h}$, meaning that for any $i\in[N_h]$, either $\mathcal S^h_i \subseteq \mathcal I$ or $\mathcal S^h_i \subseteq \mathcal R$ or $\mathcal S^h_i \subseteq \mathcal T \cap \mathcal R^c$, thus avoiding possible confusion in classifying the states after discretization. Now, for any $i\in N_h$, let $s^h_i$ be a representative element of $\mathcal S^h_i$ and define $\tilde{\mathcal S}^h \coloneqq\{s_i^h\}_{i=1}^{N_h}$ be the set of all such representatives. Moreover, we arrange $\tilde{S}^h_{\mathcal I}\coloneqq \{s_i^h\}_{i=1}^{N_{h,\mathcal I}}$ to be the representatives contained in $\mathcal I$.

Consequently, for any $x\in\mathcal S$, there exists a unique $i_x\in[N_h]$ such that $x\in\mathcal S_{i_x}^h$ and a corresponding representative element $s^h_{i_x}$. Let $\sigma(\cdot) :\mathcal S\rightarrow \tilde{\mathcal S}^h$ be the function that maps $x$ to $s^h_{i_x}$.

Now we are in a position to "discretize" the reward $r: \mathcal S\times \mathcal S\rightarrow \mathbb R$ and transition probability $p: \mathcal S\times \mathcal S\rightarrow \mathbb R$. Let $\lambda$ be the Lebesgue measure on $\mathcal S$. Define:
\begin{align*}
    \tilde{r}_h(x,y)=& r(\sigma(x),\sigma(y)) \nonumber\\
    \tilde{p}_h(x,y)=&\frac{p(\sigma(x),\sigma(y))}{\int_{S}p(\sigma(x),\sigma(z))\lambda(dz)} 
\end{align*} 

Let $\gamma\in(0,1)$ be a discount factor (typically close to 1). Define the following operators (Bellman operator from a MDP standpoint) from $\mathbb R^{\mathcal S}$ to $\mathbb R^{\mathcal S}$:
\begin{align*}
    Tv(x)=&\int_{S} (r(x,y)+v(y))p(x,y)\lambda(dy),\nonumber\\
    T_\gamma v(x)=& \int_{S} (r(x,y)+\gamma v(y))p(x,y)\lambda(dy),\nonumber\\
    \tilde{T}^h v(x)=&\int_{S} (\tilde{r}_h(x,y)+v(y))\tilde{p}_h(x,y)\lambda(dy),\nonumber\\
    \tilde{T}^h_\gamma v(x)=&\int_{S} (\tilde{r}_h(x,y)+\gamma v(y))\tilde{p}_h(x,y)\lambda(dy).
\end{align*}

For $\gamma\in(0,1)$, the operator is known to be a contraction and a unique fixed point exist (i.e. $Tv=v$). For the operator without discount factor, a unique fixed point exists under regularity conditions \cite{ahamed2006adaptive,puterman1990markov}. We assume such conditions are met and let $v^\star$, $v^\star_\gamma$, $\tilde{v}^{h,\star}$ and $\tilde{v}^{h,\star}_\gamma$ be their respective unique fixed points. 

Let $\mathcal L_h$ be the space of functions that are constant on each $\mathcal S_h^i$. It can be shown that $\tilde{r}_h,\tilde{p}_h,\tilde{v}^{h,\star},\tilde{v}^{h,\star}_\gamma \in \mathcal L_h$. Consequently, we can view the chain on $\mathcal S$ with $\tilde{r}_h,\tilde{p}_h$ as a chain on discrete state space $\tilde{\mathcal S}^h$ with a straightforward discretized version of $\tilde{r}_h,\tilde{p}_h$ and the value target $\tilde{v}^{h,\star}$.

Now, suppose we can obtain $\tilde{v}^{h,\star}$ for this discretized chain but we are interested in estimating $v^\star$, we would need to analyze $\|v^\star-\tilde{v}^{h,\star}\|_\infty$.
\begin{lemma}\label{discounterror}
For any given $\epsilon>0$, there exists some $h>0$ such that 
\begin{equation*}
   \|v^\star-\tilde{v}^{h,\star}\|_\infty < \epsilon. 
\end{equation*}
\end{lemma}
\begin{proof}
It follows from Lemma \ref{discount} we can find $\gamma$ close enough to 1 such that
\begin{align*}
   \|v^\star-{v}^\star_\gamma\|_\infty <& \frac{\epsilon}{3} \nonumber\\
    \|\tilde{v}^{h,\star}-\tilde{v}^{h,\star}_\gamma\|_\infty <& \frac{\epsilon}{3}
\end{align*}
for any $h>0$. It then follows from our discretization of the chain and \cite{chow1991optimal} (the proof is for MDP but it follows for Markov chain as well, as we can always conceptualize Markov chain as a MDP with only one action at each state) that there is some constant $C>1$ such that 
\begin{equation*}
    \|v^\star_\gamma-\tilde{v}^{h,\star}_\gamma\|_\infty \leq \frac{C}{1-\gamma} h,
\end{equation*}
where $C$ does not depend on $\gamma$ or $h$. Thus, if we pick $h$ small enough such that $h<\frac{1-\gamma}{3C}\epsilon$, then
\begin{equation*}
    \|v^\star-\tilde{v}^{h,\star}\|_\infty \leq \|v^\star-{v}^\star_\gamma\|_\infty+\|v^\star_\gamma-\tilde{v}^{h,\star}_\gamma\|_\infty +\|\tilde{v}^{h,\star}_\gamma-\tilde{v}^{h,\star}\|_\infty < \epsilon.
\end{equation*}
\end{proof}

We define 
\begin{equation}\label{error1}
\epsilon_h\coloneqq \|v^\star-\tilde{v}^{h,\star}\|_\infty    
\end{equation}
 as the discretization. We discuss in the final section on how this parameter affect the convergence of APE. After we can control $\epsilon_h= \|v^\star-\tilde{v}^{h,\star}\|_\infty$, we wanted to analyse the error bound between $\tilde{v}^{h,\star}$ and our updated  $v_{\psi^{(n)}}$ derived from \eqref{eq:ape_is_weight} and \eqref{eq:ape_v_update_approx}.

\subsection{Approximate value iteration}
As mentioned before, we denote the sequence of adaptive sampling distribution (derived from $\pi_E$ updates) to be $\{p_{\theta^{(n)}}\}_{n\geq 0}\subseteq \{p_\theta\}_\theta$ where $\theta^{(n)}\in\mathcal F_n$. If we have picked the discretization parameter $h$ sufficiently small, and our parametric or non-parametric function approximators can guarantee to be uniformly Lipschitz:
\begin{equation}\label{Lip}
    \sup_{\substack{\psi\in\Psi \\ \|x-y\|_\infty\leq h}} \frac{|v_{\psi}(x)-v_\psi(y)|}{h} < L
\end{equation}
for some $L<\infty$. Under such assumption, we want to show the update $\{v_{\psi^{(n)}}\}_{n\geq 0}$ from \eqref{eq:ape_v_update_approx} in the continuous space $\mathcal S$ can be viewed approximately as update towards $\tilde{v}^h$ on $\tilde{\mathcal S}^h$. To be specific, we again invoke the discrete Markov chain view with $(\tilde{S}^h,\tilde{p}_h,\tilde{r}_h)$ (technically the discretized version of $\tilde{p}_h,\tilde{r}_h$ only defined on $\tilde{S}^h\times\tilde{S}^h$). Suppose, on this discrete chain, we are using the following asynchronous stochastic approximation updates with given adaptive importance sampling $\{\tilde{p}_h^{(n)}\}_{n\geq 0}$ and learning rate $\alpha_n$:
\begin{align}\label{noapp}
    & \tilde{v}^{h,(n+1)}(s_i^h)=\tilde{v}^{h,(n)}(s_i^h)+ \alpha_n \nonumber \\
    & \quad \quad \mathbb I_{\{s_n=s_i^h\}}(\tilde{T}^h_i\tilde{v}^{h,(n)}-\tilde{v}^{h,(n)}(s_i^h)+M_i(n+1))
\end{align}
for all $i\in[N_{h,\mathcal I}]$ (we are only interested in states in $\tilde{S}^h_{\mathcal I}$ since $\tilde{v}^{h}(s)=0$ for $s\in\mathcal T$), where we slightly abuse notation by letting $\tilde{T}^h$ be the discrete version of the previously defined $\tilde{T}^h$ (this operator was previously $\mathbb R^{\mathcal S}\rightarrow\mathbb R^{\mathcal S}$ but now is restricted to $\mathbb R^{N_{h,\mathcal I}}\rightarrow\mathbb R^{N_{h,\mathcal I}}$):
\begin{align*}
    &\tilde{T}^h=[\tilde{T}^h_1,..., \tilde{T}^h_{N_{h,\mathcal I}}]^T  \nonumber\\
    &\tilde{T}^h_i([x_1,...,x_{N_{h,\mathcal I}}])=\sum_{j\in[N_{h,\mathcal I}]}\tilde{p}_h(s_i^h,s_j^h)x_j \\
    & \quad \quad + \sum_{j\in[N_h]} \tilde{p}_h(s_i^h,s_j^h)\tilde{r}_h(s_i^h,s_j^h)
\end{align*}
for $i\in[N_{h,\mathcal I}]$ and we let 
\begin{align}\label{Mupdate}
M(n+1)=&[M_1(n+1),...,M_{N_{h,\mathcal I}}(n+1)]\nonumber\\
 M_i(n+1)=&(\tilde{r}_h(s_n,s_{n+1})+\tilde{v}^{h,(n)}(s_{n+1}))\frac{\tilde{p}_h(s_n,s_{n+1})}{\tilde{p}_h^{(n)}(s_n,s_{n+1})} \nonumber  \\
& -\tilde{T}^h_i\tilde{v}^{h,(n)}
\end{align}
for $i\in[N_{h,\mathcal I}]$. It then follows from \cite{ahamed2006adaptive} that $\mathbb E[M(n+1)|\mathcal F_n]=0$ and $\|M(n+1)\|_\infty \leq K(1+\|\tilde{v}^{h,(n)}\|_\infty)$ for some $K>0$. Then it follows also as in \cite{ahamed2006adaptive} (from Perron-Frobenius theorem) that $\tilde{T}^h$ is a contraction w.r.t the norm induced from the eigenvector $w\in\mathbb R^{N_{h,\mathcal I}}$ of the transition matrix $\tilde{p}_h$ (restricted on $\tilde{\mathcal S}^h_\mathcal I$): 
\begin{equation}\label{wvector}
    \|\tilde{T}^h(x)-\tilde{T}^h(y)\|_w \leq \alpha \|x-y\|_w
\end{equation}
where $\|x\|_w=\max_{i\in N_{h,\mathcal I}}|\frac{x_i}{w_i}|$ and $\alpha\in (0,1)$. It follows from Theorem 1 of \cite{ahamed2006adaptive} that $\lim_{n\rightarrow\infty}\tilde{v}^{h,(n)}=\tilde{v}^h$ almost surely on $\tilde{S}^h_\mathcal I$. 

Next, to analyse the effect for parameterized value function in \eqref{noapp}, we introduce approximations of the Bellman operator from Approximate Value Iteration (AVI) as in \cite{ramaswamy2017analyzing}, which is a stochastic iterative of value iteration that are asymptotically biased/noisy and we denote this approximation operator by $A$. The use of $A$ allows a generic framework for analysis since a variety of approximations for Bellman operator exists, as we have to use approximation of objective function in high dimensions such as DNN, ANN in Deep RL or linear function approximators, or Gaussian Process and so on. Typically, if we rewrite \eqref{noapp} in a general form (not necessarily restricted to asynchronous update anymore)
\begin{equation*}
    \tilde{v}^{h,(n+1)}=\tilde{v}^{h,(n)}+\alpha_n(\tilde{T}^h\tilde{v}^{h,(n)}-\tilde{v}^{h,(n)}+M_{n+1}),
\end{equation*}
then, the additional operator $A$ in a parameterized setting ( in this case also discretized on $\tilde{\mathcal S}^h$ but not in general) gives update 
\begin{equation*}
    \tilde{v}^{h}_{\psi^{(n+1)}}=\tilde{v}^{h}_{\psi^{(n)}}+\alpha_n(A\tilde{v}^{h}_{\psi^{(n)}}-\tilde{v}^{h}_{\psi^{(n)}}+M_{n+1}),
\end{equation*}
or more generally 
\begin{equation*}
    \tilde{v}^{h}_{\psi^{(n+1)}}=\tilde{v}^{h}_{\psi^{(n)}}+\alpha_n(\tilde{T}^h\tilde{v}^{h}_{\psi^{(n)}}-\tilde{v}^{h}_{\psi^{(n)}}+\epsilon_n+M_{n+1}),
\end{equation*}
where $\epsilon_n$ represents the approximation error $A\tilde{T}^h\tilde{v}^{h}_{\psi^{(n)}}-\tilde{T}^h\tilde{v}^{h}_{\psi^{(n)}}$ but also could include additional noise term. For example, if in our update \eqref{Mupdate} we were to use $\frac{p(s_n,s_{n+1})}{p_{\theta^{(n)}}(s_n,s_{n+1})}$ similar to \eqref{eq:ape_is_weight} instead of $\frac{\tilde{p}_h(s_n,s_{n+1})}{\tilde{p}_h^{(n)}(s_n,s_{n+1})}$, we can still bound their difference by $O(h)$ term using Assumption \ref{as1} and Definition \ref{df1} which we may also incorporate into $\epsilon_n$. 

On the other hand, depending on the situation, $A$ would be different. For example, in linear function approximators, $A$ would be a projection operator that map the product of Bellman operator into the spaces spanned by the finite-dimensional basis functions. For DNN method in APE, it involves online learning and approximating bellman operator by updating the parameters which also shares a similar flavor as \eqref{eq:NF_sgd}. A strong condition which controls the difference between AVI and VI is to require the (AV3) condition in \cite{ramaswamy2017analyzing}:
\begin{assumption}\label{AV3}
\begin{equation}
    \limsup_{n\rightarrow\infty} \|\epsilon_n\|_w\leq \epsilon_p,
\end{equation}
for some fixed $\epsilon_p>0$.
\end{assumption}
 Other conditions (AV1)-(AV5) in \cite{ramaswamy2017analyzing} have all been discussed in the previous sections. It then follows from Theorem 3 in \cite{ramaswamy2017analyzing} that AVI converges to some fixed point $\tilde{v}^{h}_{\psi^{\star}}$ satisfying
 \begin{equation}\label{avibound}
     \|\tilde{T}^h\tilde{v}^{h}_{\psi^{\star}}-\tilde{v}^{h}_{\psi^{\star}}\|_w\leq \epsilon_p.
 \end{equation}
 Notice since $\tilde{T}^h$ is a contraction w.r.t $\|\cdot\|_w$ norm and $\tilde{v}^{h,\star}$ is the fixed point of $\tilde{T}^h$, i.e., $\tilde{v}^{h,\star}=\tilde{T}^h\tilde{v}^{h,\star}$. Thus,
 \eqref{avibound} implies 
 \begin{equation}\label{avilimit}
     \lim_{\epsilon_p\rightarrow 0}\|\tilde{v}^{h}_{\psi^{\star}}-\tilde{v}^{h,\star}\|_w \rightarrow 0
 \end{equation}
 Thus, if we take the view that our update is a special situation where $h$ is taken to be extremely small along with \eqref{Lip}, we may view our constructed update $\tilde{v}^h_{\psi^{(n)}}$ as a good approximation for the $v_{\psi^{(n)}}$ derived from \eqref{eq:ape_v_update_approx}:
 \begin{equation}\label{approx}
     v_{\psi^{(n)}}\approx\tilde{v}^h_{\psi^{(n)}}.
 \end{equation}
However, for the analysis of $\tilde{v}^h_{\psi^{(n)}}$, under the assumptions aforementioned, we have Lemma \ref{discounterror}, \eqref{avibound} \eqref{avilimit} and \eqref{approx} which suggest that
\begin{equation}\label{final}
    \lim_{n\rightarrow\infty}\tilde{v}^h_{\psi^{(n)}}\approx v^\star.
\end{equation}

\subsection{Concluding remarks on convergence and practical implementation}

The convergence analysis relies on several strong assumptions and is not always met in practice. For example, it is well-known Assumption \ref{AV3} may not hold and  $Q$-learning with function approximator is not necessarily convergent. However, the analysis above provides us with valuable guidelines on how one may improve the practical implementation of APE. We summarize them as follows.
\begin{itemize}
    \item The $\epsilon_p$ in Assumption \ref{AV3} depends on how expressive $\{v_\psi\}_\psi$ in terms of approximating $v^\star$. For example, for function approximators, it is known in \cite{melo2007q} that $\|\tilde{v}^{h}_{\psi^{\star}}-\tilde{v}^{h,\star}\|_\infty$ is bounded by a constant times $\|\Pi\tilde{v}^{h,\star}-\tilde{v}^{h,\star}\|$ where $\Pi$ is the projection operator which maps function into the space spanned by $\{v_\psi\}_\psi$. The more expressive is $\{v_\psi\}_\psi$, the more accurate we can bound $\|\tilde{v}^{h}_{\psi^{\star}}-\tilde{v}^{h,\star}\|_\infty$. Thus, in practice, when one found APE is having a hard time exhibiting convergence, it is sensible to try a more expressive collection to approximate $v^\star$.
    
    \item In Lemma \ref{discounterror}, we saw the approximation error is related to discretization parameter $h$. Smaller $h$ leads to higher accuracy, which in some sense justifies our implementation of APE which directly updates the value on a continuous state $\mathcal S$. It can be viewed as our artificial update \eqref{noapp} on an infinitely small discretization $h$. However, to make the numerical update more robust, we sample a buffer $\mathcal D$ and aggregate our value update as opposed to the update \eqref{noapp}.
    
    \item The choice adaptive sampling distribution is not directly mentioned in the analysis above. However, the approximation error in Lemma \ref{discounterror} also has a implicit dependence on a hidden parameter $\gamma$ which determines how close is $\mathbb E[\sum_{t=1}^{\tau}\gamma^{t-1}r(s_{t-1},s_t)]$ from $\mathbb E[\sum_{t=1}^{\tau}r(s_{t-1},s_t)]$. The difference between these two quantities largely depend on how fast the tail of $\tau $ diminishes. The almost only effective way to control this is to adaptively choose importance sampling distribution approaching the "zero-variance" sampling distribution $p^\star$.
    
\end{itemize}

\section*{ACKNOWLEDGMENT}

The authors gratefully acknowledge the support from the National Science Foundation (under grants IIS-1849280 and IIS-1849304), and unrestricted research grant from the Toyota Motor North America. The ideas, opinions and conclusions presented in this paper are solely those of the authors.



\bibliographystyle{IEEEtran} 
\bibliography{ape}